%% file: main.tex
\documentclass{article}

\usepackage{microtype}
\usepackage{graphicx}
\usepackage{booktabs} 

\usepackage[accepted]{icml2023}

\usepackage{amsmath}
\usepackage{amssymb}
\usepackage{mathtools}
\usepackage{amsthm}

\theoremstyle{plain}
\newtheorem{theorem}{Theorem}[section]

\theoremstyle{definition}

\theoremstyle{remark}

\usepackage[textsize=tiny]{todonotes}

\usepackage[utf8]{inputenc} 
\usepackage[T1]{fontenc}    
\usepackage{hyperref}       
\usepackage{amsfonts}       
\usepackage{nicefrac}       
\usepackage{xcolor}         
\usepackage{subcaption}
\usepackage{wrapfig}
\usepackage{siunitx}

\usepackage{listings}
\usepackage[british]{babel}

\input{math_commands.tex}

\hypersetup{colorlinks=false, allcolors=black, pdfborder={0 0 0}}
\usepackage{url}
\usepackage[capitalise, noabbrev]{cleveref}

\usepackage{multirow, makecell}
\usepackage{soul}

\hypersetup{  
     colorlinks   = true,
     citecolor    = gray
}

\definecolor{attnmask}{rgb}{0.8, 0.38, 0.38}
\definecolor{nodeemb}{rgb}{0.74, 0.49, 0.8}
\definecolor{observed}{rgb}{0.47, 0.8, 0.42}
\definecolor{onehot}{rgb}{0.42, 0.63, 0.8}
\definecolor{intermediate}{rgb}{0.75, 0.8, 0}

\begin{document}

\twocolumn[
\icmltitle{Graphically Structured Diffusion Models}

\begin{icmlauthorlist}
\icmlauthor{Christian Weilbach}{comp}
\icmlauthor{William Harvey}{comp}
\icmlauthor{Frank Wood}{comp}
\end{icmlauthorlist}

\icmlaffiliation{comp}{Department of Computer Science, University of British Columbia, Vancouver, Canada}

\icmlcorrespondingauthor{Christian Weilbach}{weilbach@cs.ubc.ca}

\icmlkeywords{Machine Learning, ICML}

\vskip 0.3in
]

\printAffiliationsAndNotice{}

\begin{abstract}
We introduce a framework for automatically defining and learning deep generative models with problem-specific structure. We tackle problem domains that are more traditionally solved by algorithms such as sorting, constraint satisfaction for Sudoku, and matrix factorization. Concretely, we train diffusion models with an architecture tailored to the problem specification. This problem specification should contain a graphical model describing relationships between variables, and often benefits from explicit representation of subcomputations. Permutation invariances can also be exploited. Across a diverse set of experiments we improve the scaling relationship between problem dimension and our model's performance, in terms of both training time and final accuracy. Our code can be found at \url{https://github.com/plai-group/gsdm}.
\end{abstract}

\section{Introduction}\label{sec:introduction}


A future in which algorithm development is fully transformed from a  challenging and labour intensive task \citep{cormenIntroductionAlgorithms2009, marslandMachineLearningAlgorithmic2009, stuartrussellArtificialIntelligenceModern2010, williamsonDesignApproximationAlgorithms2011a} into a fully automatable process is seemingly close at hand.  With prompt engineering large language models like GPT~\citep{brownLanguageModelsAre2020} and now ChatGPT have been shown to be capable of code completion and even full algorithm development from natural language task descriptions~\citep{chenEvaluatingLargeLanguage2021,ouyang2022training}.   

At the same time, significant advances in generative deep learning \citep{hoDenoisingDiffusionProbabilistic2020a, songScoreBasedGenerativeModeling2021, hoVideoDiffusionModels2022}, AutoML \citep{hutterAutomaticMachineLearning2018}, and few-shot learning \citep{brownLanguageModelsAre2020} have made it possible to learn, from data, flexible input-output mappings that generalize from ever smaller amounts of data.  This approach has spawned modern aphorisms from \citet{karpathy_2017} like ``
Gradient descent can write better software than you.  Sorry!'',
appropriate attempts, in our opinion, to re-brand deep learning as differentiable programming \citep{baydinAutomaticDifferentiationMachine2017}, and arguably even a new industry called ``Software 2.0'' in which one ``programs by example''~\citep{karpathy_2017}.

There, however, remains a chasm between these two approaches, roughly delineated along the symbolic vs.~connectionist divide.  Symbolically expressed algorithms can and often do generalize perfectly across all inputs and exhibit runtimes that are typically input ``size'' dependent.  Software 2.0 algorithms struggle to generalize outside of their training data thus are most often deployed in settings where copious training data is available, the so-called ``big-data'' regime.  Most such ``neural-network algorithms'' have runtimes that are not size dependent and resultingly cannot generalize in the same fashion as symbolically expressed algorithms.

Efforts to bring these two approaches closer together \citep{chaudhuriNeurosymbolicProgramming2021} often get lumped together under the moniker ``neuro-symbolic'' methods.  The general shape of these methods, so to speak, is to impose some aspect of symbolic reasoning on either the structure or computation performed by a connectionist architecture.  Our work can be seen as a significantly novel methodological contribution to this body of work.

We contribute a generic specification of methodology for advantageously imposing task specific symbolic structure into diffusion models and use it to demonstrate algorithm learning from data in several small-scale but foundational tasks across the algorithmic complexity spectrum. Specifically, our approach consumes a graphical model ``sketch'' that putatively could describe the joint data generative process. This sketch consists only of nodes for variables, edges between them, and optionally permutation invariances. We combine this information with an otherwise generic diffusion process~\citep{hoDenoisingDiffusionProbabilistic2020a}, using the edges to advantageously constrain the transformer attention mechanisms~\citep{vaswaniAttentionAllYou2017a} and permutation invariances to determine when parameters within our architecture can be shared. Compared to our neural baselines we improve the scaling of computational cost with problem dimension in most cases, and the scaling of problem performance with dimension in all cases.

\begin{figure*}[ht]
    \centering
    \includegraphics[width=0.95\textwidth]{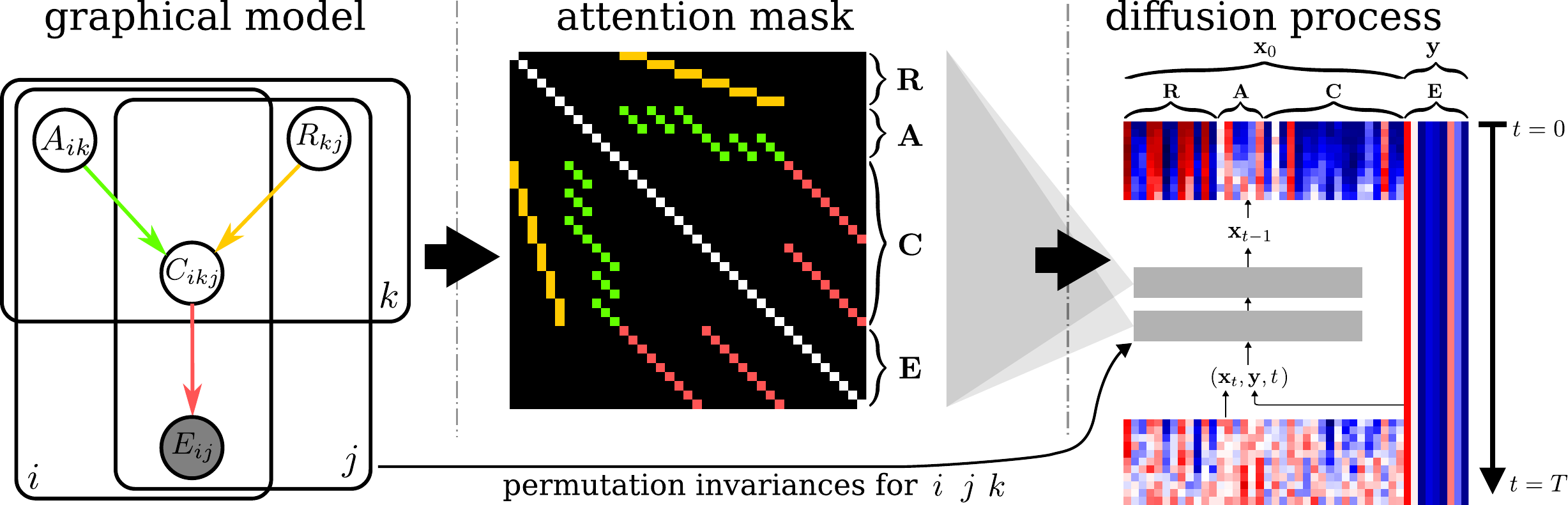}
    \caption{An application of our framework to binary-continuous matrix factorization. 
    In the first panel the computational graph of the multiplication of the continuous matrix $\mA\in\sR^{3\times2}$ and the binary matrix $\mR\in\sR^{2\times3}$ is expanded as a probabilistic graphical model in which intermediate products $\mC$ are summed to give $\mE=\mA\mR$. This graph is used to create a structured attention mask $\mM$, in which we highlight 1's with the color of the corresponding graphical model (or white for self-edges). 
    In the third panel the projection into the sparsely-structured neural network guiding the diffusion process is illustrated. The bottom shows the translation of permutation invariances of the probability distribution into shared embeddings, as detailed in \cref{sec:permutation-invariance}.}
    \label{fig:bmf_translation}
\end{figure*}

As a running example to keep in mind throughout the paper, consider being given the task of developing a novel matrix factorization algorithm, one which takes a real non-negative valued matrix as input and outputs a distribution over two matrix factors, one constrained to be binary valued, the other constrained to be non-negative.  The most traditional approach is to painstakingly hand-develop through intellectual willpower some new algorithm like Gram-Schmidt which may not exist and might take an entire career to develop.  A more modern approach, to which we compare, is to symbolically specify a model describing a joint data generating process and employ a generic inference algorithm like MCMC~\cite{wingateLightweightImplementationsProbabilistic}. Such a model is usually much easier to specify but the resulting ``inversion algorithm,'' running a generic inference algorithm at test time, trades sure generalization with worst-case infinite runtime.   Alternatively one could  generate a large training dataset from such a generative description, then hand-architect and train a deep neural network to learn the desired inversion algorithm, software 2.0 style \citep{le2017using}.  This is slow to develop and train, usually requiring architectural innovation, but constant-time fast at test time, albeit with likely poor algorithm-style generalization.  Our approach, most like  that of \citep{weilbachStructuredConditionalContinuous2020}, strikes a middle ground.   We adopt the  software 2.0 approach but provide a generic recipe for specializing a generic and powerful diffusion-based network architecture that trains quickly, generalizes reliably, and whose runtime scales with problem size.

\section{Background}\label{sec:background}

\subsection{Conditional Diffusion Models}\label{sec:dms}
Defining $\rvx_0$ to be data sampled from a data distribution $q(\rvx_0)$, a diffusion process constructs a chain $\rvx_{0:T}$ with noise added at each stage by the transition distribution
\begin{equation}\label{eq:diffusion-forward}
    q(\rvx_t | \rvx_{t-1}) = \gN(\rvx_t ; \sqrt{1-\beta_t}  \rvx_{t-1} , \beta_t \rmI )
\end{equation}
leading to the joint distribution
\begin{equation}\label{eq:diffusion-transition}
    q(\rvx_{0:T}) = q(\rvx_0) \prod_{t=1}^T q(\rvx_t | \rvx_{t-1}).
\end{equation}
The sequence $\beta_{1:T}$ controls the amount of noise added at each step and, along with $T$ itself, is chosen to be large enough that the marginal $q(\rvx_T|\rvx_0)$ resulting from \cref{eq:diffusion-forward} is approximately a unit Gaussian for any $\rvx_0$.

This diffusion process inspires a diffusion model~\citep{sohl-dicksteinDeepUnsupervisedLearning2015,hoDenoisingDiffusionProbabilistic2020a,songScoreBasedGenerativeModeling2021}, or DM, which approximately ``inverts'' it using a neural network that outputs $p_\theta(\rvx_{t-1}|\rvx_t) \approx q(\rvx_{t-1}|\rvx_t)$. We can sample from a diffusion model by first sampling $\rvx_T\sim p(\rvx_T)=\gN(\mathbf{0},\mathbf{I})$ and then sampling $\rvx_{t-1} \sim p_\theta(\cdot|\rvx_t)$ for each $t=T,\ldots,1$, eventually sampling $\rvx_0$. In the conditional DM variant~\citep{tashiro2021csdi} the neural network is additionally conditioned on information $\rvy$ so that the modelled distribution is
\begin{equation}\label{eq:diffusion-reverse}
    p_\theta(\rvx_{0:T}|\rvy) = p(\rvx_T) \prod_{i=1}^T p_\theta(\rvx_{t-1}|\rvx_t, \rvy).
\end{equation}
The transitions $p_\theta(\rvx_{t-1}|\rvx_t, \rvy)$ are typically approximated by a Gaussian with non-learned diagonal covariance, and so the learning problem is simply to fit the Gaussian's mean. \citet{hoDenoisingDiffusionProbabilistic2020a} parameterize this mean as an affine function of $\mathbb{E}[\rvx_0|\rvx_t,\rvy]$ and, by doing so, reduce the problem to fitting an estimator of $\rvx_0$ from $\rvx_t$ and $\rvy$ with the loss 
\begin{equation}\label{eq:diffusion-loss}
    \mathcal{L}(\theta) = \sum_{t=1}^T \mathbb{E}_{q(\rvx_0, \rvx_t, \rvy)} \left[ \lambda(t) \cdot \lVert \hat{\rvx}_\theta(\rvx_t, \rvy, t) - \rvx_0 \rVert_2^2 \right].
\end{equation}
\citet{hoDenoisingDiffusionProbabilistic2020a,songMaximumLikelihoodTraining2021a} show that there exists a weighting function $\lambda(t)$ such that this loss is (the negative of) a lower-bound on the marginal log-likelihood $\log p_\theta(\rvx_0|\rvy)$. We instead use uniform weights $\lambda(t) = 1$ which has been shown to give better results in practice~\citep{hoDenoisingDiffusionProbabilistic2020a}.

\subsection{Transformer Architecture}
\label{sec:transformer}

\cref{fig:attention_circuit} outlines our neural architecture for $\hat{\rvx}_\theta$, which is run once for every diffusion time step $t$. Its inputs are $\rvx_t$, $\rvy$, and the diffusion timestep $t$. The (noisy) value of each latent graphical model node is represented in $\rvx_t$ and, similarly, $\rvy$ contains the value of each observed graphical model node. We use linear projections to embed all of these values, concatenating the embeddings of all latent and observed nodes to obtain, for an $n$-node graphical model and $d$-dimensional embedding space, an $n \times d$ array of embedding ``tokens''.
We add learned ``node embeddings'' to these tokens to identify which graphical model node each corresponds to, and also add learned observation embedding vectors for tokens corresponding to observed nodes.
The resulting $n \times d$ array is passed through a stack of self-attention~\citep{vaswaniAttentionAllYou2017a} and ResNet~\citep{heDeepResidualLearning2016} blocks, as summarized in Figure 2, with the ResNet blocks taking an embedding of the diffusion timestep $t$ as an additional input. All of the timestep embedder, ResNet blocks, and self-attention modules are identical to those of \citet{songDenoisingDiffusionImplicit2021}, except that we replace convolutions with per-token linear projections.
The tokens corresponding to non-observed graphical model nodes are then fed through a learned linear projection to produce an output for each.

The self-attention layers are solely responsible for controlling interactions between embeddings, and therefore correlations between variables in the modelled distribution $p_\theta(\rvx_0|\rvy)$. 
Inside the self-attention, each embedding is projected into a query vector, a key vector, and a value vector, all in $\mathbb{R}^d$. Stacking these values for all embeddings yields the matrices $\mQ,\mK,\mV \in \mathbb{R}^{n \times d}$ (given a $n$-node graphical model). The output of the self-attention is calculated as
\begin{align}
\label{eq:attention}
    \emb^\text{out} &= \emb^\text{in} + \mW \mV = \emb^\text{in} + \softmax\left( \mQ \mK^{T} \right) \mV
\end{align}
where the addition of the self-attention layer's input $\emb^\text{in}\in\mathbb{R}^{n \times d}$ corresponds to a residual connection. We note that $\mQ \mK^{T}$ yields a pairwise interaction matrix which lets us impose an additional attention mask $\mM$ before calculating the output $\emb^\text{out} = \emb^\text{in} + \softmax\left( \mM \odot \mQ \mK^{T} \right) \mV$. This masking interface is central to structure the flow of information between graphical model nodes in \Cref{sec:structured-attention} .

\subsection{Graphical Models} \label{sec:graphical-model}
GSDM leverages problem structure described in the form of a graphical model. There is considerable flexibility in the specification of this graphical structure and we allow for both directed and undirected graphical models. A directed graphical model describes a joint distribution over ${\mathbf{x}=[x_1,\ldots,x_n]}$ with the density ${p(\mathbf{x})=\prod_{i=1}^n p_i(x_i|parents(x_i))}$. This may be natural to use if the problem of interest can be described by a causal model. This is the case in the BCMF example in \cref{fig:bmf_translation}, where the forward model is a matrix multiplication and we can use the matrix multiplication's compute graph as a graphical model. If the data is simulated and source code is available then we can automatically extract the simulator's compute graph as a graphical model (\Cref{app:bmf_compilation}). Alternatively, an undirected graphical model uses the density ${p(\mathbf{x}) \propto \prod_{j=1}^m f_j(vertices(j))}$ where $vertices(j)$ are the vertices connected to factor $j$ and $f_j$ maps their values to a scalar. This
is a natural formulation if the problem is defined by constraints on groups of nodes, e.g. for Sudoku with row, column and block constraints (\cref{app:attention_masks}). Finally, the graphical model can combine directed and undirected components, using a density ${p(\mathbf{x}) \propto \prod_{i=1}^n p(x_i|parents(x_i)) \prod_{j=1}^m f_j(vertices(j))}$. We use this in our graphical model for sorting (\cref{app:attention_masks}), which combines a causal forward model with constraints.

We emphasise that GSDM does not need the link functions (i.e. the form of each $p_i$ and $f_j$) to be specified as long as data is available, which is desirable as they are often intractable or Dirac in practice.
Also, while the selection of a graphical model for data can be subjective, we find in \cref{sec:model-structure-ablations} that GSDM is not sensitive to small changes in the specification of the graphical model and that there can be multiple modeling perspectives yielding similar GSDM performance. In general, we use the most intuitive graphical model that we can come up with for each problem whether it is directed, undirected, or a combination.

\subsection{Permutation Invariance} \label{sec:permutation-invariance}
Large probabilistic models often contain permutation invariance, in the sense that the joint probability density $q(\rvx_0)$ is invariant to permutations of certain indices~\citep{bloem-reddyProbabilisticSymmetriesInvariant2020}. For example the matrix multiplication in \cref{fig:bmf_translation} is invariant with respect to permutations of any of the plate indices.\footnote{In general, plate notation implies permutation invariance as long as no link functions depend on the plate indices themselves.} If the joint probability density is invariant to a particular permutation, this can be enforced in a distribution modelled by a DM by making the neural network architecture equivariant to the same permutation~\citep{hoogeboom2022equivariant}. We show how to encode such equivariances in GSDM in \cref{sec:varying_dimensions}.

\begin{figure}[t]
    \centering
    \includegraphics[width=0.9\columnwidth]{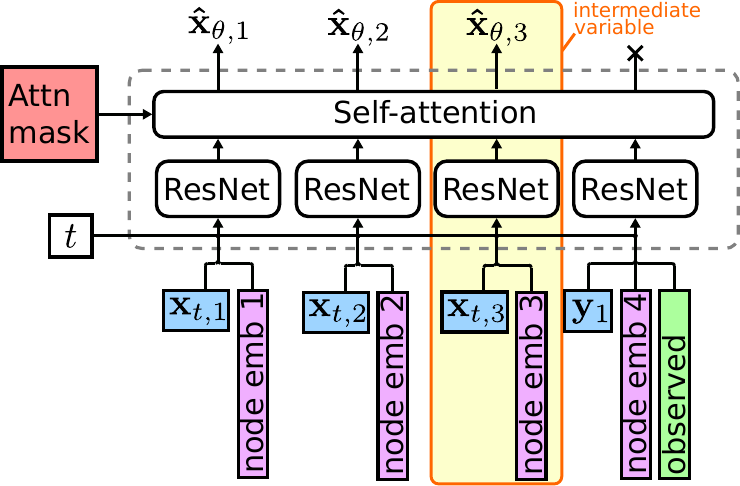}
    \caption{An example GSDM architecture for a graphical model with one observed and three latent variables. The components within the dashed lines are repeated multiple times. Arrows represent information flow. For clarity, we leave out simple operations including linear transformations; see the appendix for full detail.}
    \label{fig:attention_circuit}
\end{figure}
\section{Method}
\label{sec:methods}

The first stage in using GSDM is to define a graphical model as discussed previously.
This section focuses on how to map from a graphical model to the corresponding GSDM architecture, an example of which is shown in \cref{fig:attention_circuit}. The backbone of the architecture is a stack of transformer modules operating on a set of embeddings, with each embedding corresponding to one graphical model node. The colors in \cref{fig:attention_circuit} outline this section.
\textcolor{attnmask}{\Cref{sec:structured-attention} describes how we derive a sparse attention mechanism from a graphical model.}
\textcolor{nodeemb}{\Cref{sec:node-embeddings} explains the node embeddings.}
\textcolor{intermediate}{\Cref{sec:intermediate_variables} motivates our decision to model ``intermediate'' variables jointly with the variables of interest.}
\textcolor{observed}{\Cref{sec:flexible_conditioning} describes how we handle observed variables}, and \textcolor{onehot}{\cref{sec:mixed_continous_discrete_variables} describes how we handle discrete variables.}

\subsection{Faithful Structured Attention} \label{sec:structured-attention}
Our architecture in \cref{fig:attention_circuit} runs a self-attention mechanism over a set of embeddings, each of which corresponds to a graphical model node. To add structural information, we use the graphical model's edges to construct attention masks $\mM$ for the self-attention layers. Precisely, we allow variable $i$ to attend to variable $j$ if there is an edge between node $i$ and node $j$, and irrespective of the direction of the edge. If the graphical model contains factors, we additionally allow attention between any node pairs $(i,j)$ which connect to the same factor. These heuristics typically keep the graph sparse while, importantly, ensuring that given enough attention layers all output nodes can depend on all input nodes. We show in \cref{app:faithfulness} that this is necessary to faithfully capture the dependencies in $q(\rvx|\rvy)$.

To reduce our memory usage and computation compared to a dense matrix multiplication of the masked matrix we provide an efficient sparse attention implementation as described in \Cref{app:sparse_attention} and the released code. Its computational cost is $\gO(nm)$, where $n$ is the number of dimensions and $m$ is the number of entries in the densest row of $\mM$. We show in \cref{ap:complexities} that, even after accounting for cost of the modeling of additional intermediate variables as described later, the sparse attention mechanism gives GSDM a reduction in computational complexity of $\mathcal{O}(n)$ relative to a naive baseline in three of our four experiments.

\subsection{Node Embeddings} \label{sec:node-embeddings}
\label{sec:varying_dimensions}  

GSDM's architecture contains positional embeddings to let the neural network distinguish which inputs correspond to which graphical model nodes. The simplest variation of GSDM learns one embedding per graphical model node independently, and we call this approach \textbf{independent embeddings}, or \textbf{IE}. An issue with IE is that it cannot generally be adapted to changing problem dimension.
A generic solution to this involves noting that graphical model nodes can often be grouped together into ``arrays''. For instance, the BCMF example in \cref{fig:bmf_translation} contains 39 nodes but these belong to just 4 multi-dimensional arrays: $\mA$, $\mR$, $\mC$, and $\mE$. We suggest \textbf{array embeddings}, or \textbf{AE}, which can be automatically constructed for such problems with (potentially variable-size) ordered datatypes. With AE, we compute the embedding for each node as the sum of a shared array embedding, learned independently for every array, and a sinusoidal positional embedding~\citep{vaswaniAttentionAllYou2017a} for its position within an array. Scalars can be treated as arrays of size 1. 
AEs work well in our experiments and are a sensible default.

For graphical models exhibiting permutation invariances we can optionally enforce these invariances exactly using \textbf{exchangeable embeddings}, or \textbf{EE}. We do so according to the following result.
\begin{theorem}[Permutation invariance in GSDM]
\label{theo:perm-invariance-gsdm}
Let $\mathcal{A}$ represent the indices of a subset of the dimensions of data $\rvx$ and $\Pi_\mathcal{A}$ be the class of permutations that permute only dimensions indexed by $\mathcal{A}$. Assume we have a GSDM parameterised with neural network $\hat{\rvx}_\theta(\cdot; \mM)$, where $\mM$ is the structured attention mask. If the node embeddings used by $\hat{\rvx}_\theta$ are shared across all nodes indexed by $\mathcal{A}$, then the distribution modelled by GSDM will be invariant to all permutations $\pi$ satisfying
\begin{equation}
    \mM = \pi \mM \quad \text{and} \quad \pi \in \Pi_\mathcal{A}
\end{equation}
where $\pi\mM$ is a permutation of both the rows and columns of $\mM$ by $\pi$.
\end{theorem}
\begin{proof}
See \cref{app:permutation-invariance}.
\end{proof}

One implication of the $\mM=\pi\mM$ condition is that it holds trivially for a DM without sparse attention, in which $\mM$ is a matrix of all ones. The modeled distribution would therefore be invariant to any permutation of $\mathcal{A}$~\cite{hoogeboom2022equivariant}. This may be a useful permutation invariance to encode for some problems but, for the structured problems considered in this paper, it is too simple and not valid. In none of our experiments are there two variables whose values can be swapped without changing the density under the data distribution. For example in BCMF, the data density is invariant to reordering any of the plate indices, but not to swapping any single pair of nodes in them.

When $\mM$ is a structured matrix as we propose for GSDM, \cref{theo:perm-invariance-gsdm} suggests a way to incorporate invariances that are closely tied to the problem structure. On the BCMF problem shown in \cref{fig:bmf_translation} we use only four embeddings, sharing a single embedding between all nodes in $\mA$; another between all nodes in $\mR$; and so on for $\mC$ and $\mE$. Without imposing an attention mask, this would make the network invariant to any permutation of the variables within each of $\mA$, $\mR$, $\mC$ and $\mE$. With GSDM's attention mask, it is only invariant to permutations which lead to the same mask. This means that the learned distribution is invariant only to the ordering of the indices $i$, $j$, and $k$. As represented by the plate notation in \cref{fig:bmf_translation}, this is a desired invariance that matches the data distribution.

In general, \cref{theo:perm-invariance-gsdm} suggests a simple heuristic for checking when problem invariances can be enforced. If permuting the ordering of an index does not affect the sparsity mask for a given problem then sharing embeddings across instances of this index will enforce a permutation invariance with respect to this index in GSDM. We utilise this property in three of our four experiments. Along with our compiler which generates an attention mask for a programatically-defined graphical model, checking when \cref{theo:perm-invariance-gsdm} holds for a given class of permutations can reasonably be automated. We do, however, still require human knowledge to propose invariances suitable for the graphical model.

\subsection{Intermediate Variables}
\label{sec:intermediate_variables}
\begin{figure}[t]
    \centering
    \includegraphics[width=\columnwidth]{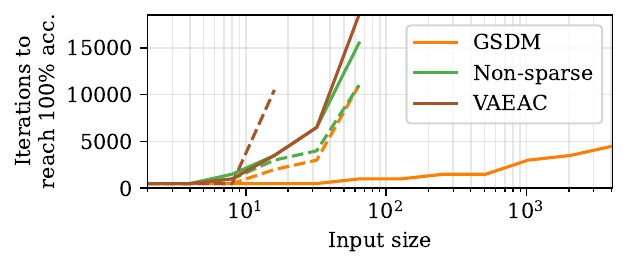}
    \caption{Time to fit the Boolean circuit with (solid lines) and without (dashed) intermediate variables. Accuracy is computed on 16 validation examples every 500 iterations, to a maximum of $20\,000$.}
    \vspace{-2mm}
    \label{fig:boolean-intermediate-vars-demonstration}
\end{figure}

When translating a generative model into a graphical model, any observed variables or latent variables of particular interest to the modeler should clearly be included as nodes. There will be other latent variables, which we call ``intermediate variables'', which are not directly of interest but may be included to make the graphical model more interpretable, more sparse, or otherwise preferable. Whether or not these are included has implications for GSDM as it will either be trained to model them jointly with the variables of interest if they are included, or trained without this signal if they are not. There is no model-agnostic ``right'' answer to whether or not intermediate variables will be helpful but we point out that including them is often beneficial for GSDMs because of \textbf{(1)} GSDM's empirical success at utilising the learning signal from these intermediate variables as described below and \textbf{(2)} the reduced computational cost of GSDM's sparse attention that is related to the number of graphical model \textit{edges} more so than the number of nodes, and so is not necessarily increased by adding intermediate variables.

As an illustrative example, consider a Boolean logic circuit which takes an input of size $2^n$. The input is split into pairs and each pair is mapped through a logic gate to give an output of size $2^{n-1}$. After $n$ layers and a total of $2^n-1$ logic gates, there is a single output. Suppose that you know that each gate is randomly assigned to be either an OR gate or an AND gate, and you wish to infer which from data. If the data contains only the inputs and the single output, it contains only 1 bit of information. Identifying the function computed by each of the $\gO(2^n)$ gates will therefore require at least $\gO(2^n)$ data points. On the other hand, if the data contains intermediate variables in the form of the output of every logic gate, each data point contains $\gO(2^n)$ bits of information so the task may be solvable with only a few data points.
\Cref{fig:boolean-intermediate-vars-demonstration} shows that this reasoning holds up when we train a DM on this example. Without intermediate variables, the number of training iterations needed scales exponentially with $n$. With the combination of intermediate variables and structured attention, however, the training behaviour is fundamentally changed to scale more gracefully with $n$.

\subsection{Flexible Conditioning}
\label{sec:flexible_conditioning}

Optimizing the DM loss in \cref{eq:diffusion-loss} requires a partitioning of data into latent variables (outputs) $\rvx_0$, and observed variables (inputs) $\rvy$. Our neural network distinguishes between variables in $\rvx_t$ and $\rvy$ via a learned observation embedding vector $\emb_o$ that is added to the embeddings of observed variables. This approach also naturally allows us to deal with missing values, or additional observed values, at inference time. This is unlike standard algorithms or more traditional autoregressive amortized inference artifacts~\citep{leInferenceCompilationUniversal2017}.

\subsection{Handling Mixed Continuous/Discrete Variables} \label{sec:mixed_continous_discrete_variables}

Our simple approach to combining discrete and continuous variables in a DM is to map the discrete variables to one-hot encodings in a continuous space before running the diffusion process. Sampled one-hot encodings can then be mapped back to the discrete space with an $\argmax$. We project the entire (diffused) one-hot encoding into a single embedding before passing it into the transformer, so that the transformer performs the same amount of computation for a discrete variable as for a continuous variable.

\begin{figure}[t]
    \centering
    \includegraphics[width=0.8\columnwidth]{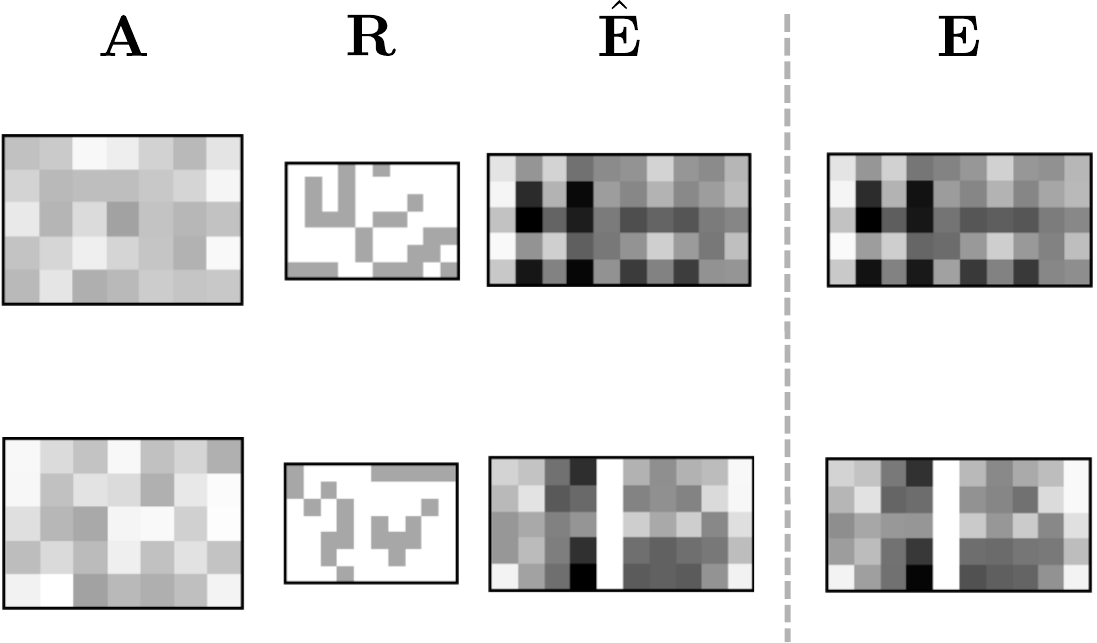}
    \caption{Two BCMF samples by GSDM with $m=5$, $n=10$ and $k=7$. Each row in the plot shows the inferred $\mA$, $\mR$ and the respective reconstruction $\hat{\mE}=\mA\mR$ on the left and the $\mE$ provided as input on the right. Intermediate variables $\mC$ are omitted here.}
    \label{fig:bmf_reconstruction}
\end{figure}

\section{Experiments}
\label{sec:experiments}

Our experiments compare GSDM against ablations including a non-sparse version (i.e. a vanilla DM), as well as the variational auto-encoder for arbitrary conditioning (VAEAC) \citep{ivanovVariationalAutoencoderArbitrary2019} and, where appropriate, the best performing MCMC method we tried: Lightweight Metropolis Hastings (LMH)~ \citep{wingateLightweightImplementationsProbabilistic}. Unlike GSDM, MCMC requires a fully specified probabilistic model and cannot operate with data and independencies only. We additionally compare against two deterministic baselines trained to predict $\rvx$ given $\rvy$ with a mean-squared error loss. One of these, ``Regressor + GS'' has a similar architecture to GSDM and the other, ``Regressor'' has a similar architecture to our vanilla DM baseline. Both perform poorly because most of our problems have multiple solutions. We test the methods on the following four problems, which were chosen to cover a wide range of problems in the space of algorithm design and demonstrate that GSDM is suitable to learn such approximate algorithms.

\paragraph{Binary continuous matrix factorization (BCMF)}
Our first experiment tackles the challenging BCMF problem, where we factorize a continuous matrix into one binary and one continuous component. Our BCMF generative model samples a binary matrix $\mR \in \R^{k \times n}$ elementwise from $\text{Bernoulli}(0.3)$ and a continuous matrix $\mA \in \R^{m \times k}$ from an elementwise $\text{Uniform}(0, 1)$ prior. The BCMF task is to infer these conditioned on $\mE := \mA \mR$. To obtain intermediate variables as discussed in \cref{sec:intermediate_variables} we break the matrix multiplication into two steps, $C_{ijk} := A_{ik}R_{kj}$ and then $E_{ij} := \sum_k C_{ijk}$. Our latent variables $\rvx_0$ are therefore the combination of all elements of $\mR$, $\mA$, and $\mC$ and the observed variables $\rvy$ are the elements of $\mE$. \cref{fig:bmf_reconstruction} shows that our learned GSDM can find factorisations that are consistent with the observations.
We plot performance for GSDM trained on different problem sizes in \cref{fig:scaling-results}. We highlight another property here, namely that GSDM trained on BCMF can generalise well to larger $m$ and $n$ than it is trained on; a GSDM trained on $m$, $n$ and $k$ uniformly sampled in the range $1$ to $10$ can generalize to problems with $m$ and $n$ as large as 200 as we show in \cref{app:conditional_bcmf}.

In keeping with our motivation of GSDM as a tool for quickly learning approximate algorithms, we compare against several quickly-implemented heuristic algorithms for BCMF. First, K-means uses the K-means algorithm to assign each row of $\mE$ to one of $k$ cluster. The factor $\mA$ is then the cluster centers and $\mR$ is a binary tensor describing the assigned cluster indices. Next, NMF uses an off-the-shelf continuous non-negative matrix factorization algorithm and discretizes one of the returned factors. Finally, we found that simply prompting ChatGPT to write a BCMF algorithm in Python yielded a non-trivial algorithm which we denote ``ChatGPT'' (see \cref{app:bcmf-baselines}). GSDM outperforms all of these baselines whenever $k$ is less than $n$ and $m$ (so that there are no trivial solutions). Additionally, we emphasize that GSDM targets the posterior over all solutions, while these hand-coded algorithms yield a single point estimate.

\begin{figure*}[!h]
    \includegraphics[scale=0.73]{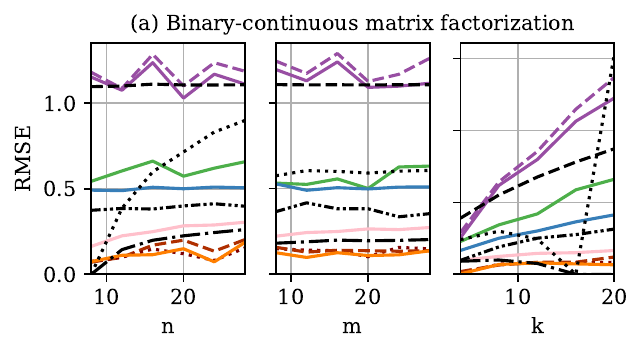}
    \includegraphics[scale=0.73]{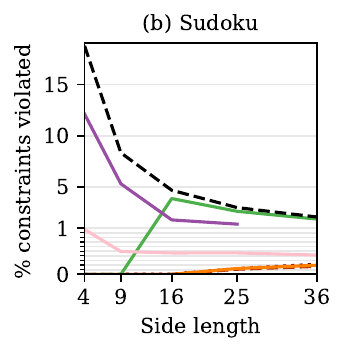}
    \includegraphics[scale=0.73]{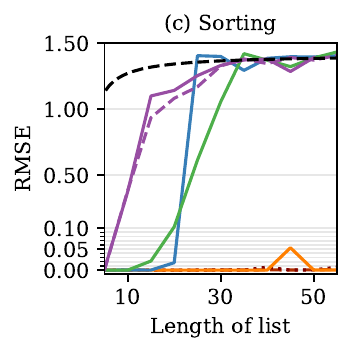}
\begin{subfigure}[b]{\textwidth}
\centering
    \includegraphics[scale=0.6]{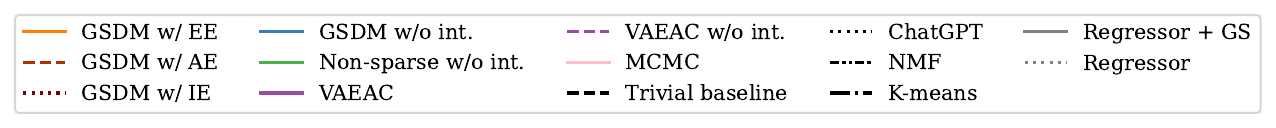}
\end{subfigure}
\caption{
Performance versus problem dimension with all runs matched for wall-clock time. \textbf{(a)} BCMF reconstruction error. We start from $n=m=16$ and $k=8$ and vary each dimension independently. GSDM is the best-performing method for all settings, and is also the only neural method to outperform our MCMC baseline. \textbf{(b)} Number of pairwise constraint violations from samples for the Sudoku task with 20\% of cells observed, normalised by the total number of constraints. Only GSDM and the MCMC baseline significantly outperform it across all tested problem dimensions. \textbf{(3)} RMSE between ground-truth sorted lists and the observed lists sorted with sampled permutation matrices. All methods work on lists of length 5 but our baselines quickly degrade in performance as the dimensionality is increased, while GSDM continues to work well. The dashed black line represents a trivial baseline given by, in (a), predicting the mean under the prior, in (b), sampling all cells from independent uniform distributions and, in (c), sampling a random permutation matrix.
We share the legend with \cref{fig:training_plots}, which contains results for ``Regressor'' and ``Regressor + GS''.
}
\label{fig:scaling-results}
\end{figure*}

\paragraph{Sudoku}
A Sudoku grid is a $9\times9$ array of numbers such that each number is in $\{1,\ldots,9\}$ and no two numbers in the same row, column, or $3\times3$ block are the same. Solving a Sudoku, i.e. completing one given a partially-filled in grid, is a difficult problem for deep learning methods, and has previously been addressed with hand-designed modules for reasoning~\citep{palmRecurrentRelationalNetworks2018} or semi-definite programming~\citep{wangSATNetBridgingDeep2019}. We use GSDM without such custom modules for combinatorial reasoning. We model a Sudoku with a factor graph. There is one factor for each  row, column, and block representing the constraint that it contains all numbers $\{1,\ldots,9\}$. The resulting GSDM attention mask lets each variable attend to all other variables in the same row, column, and block. Our data generator\footnote{We generated complete Sudokus with a port of \url{https://turtletoy.net/turtle/5098380d82}} creates complete $9\times9$ Sudokus. In order to train GSDM as a Sudoku solver for arbitrary Sudoku puzzles, we create each training example by randomly partitioning the grid into latent and observed portions by sampling $n_o$ uniformly from 0 to 80 and then uniformly sampling $n_o$ variables to observe. We find that GSDM can indeed solve Sudoku puzzles: it generates valid Sudokus unconditionally with $98\%$ accuracy; and in the case where 16 of 81 cells are observed with $96\%$ accuracy while maintaining sample diversity (\Cref{app:sudoku_sample_diversity}). We also consider a generalization of Sudokus to any $n^2 \times n^2$ grid. Our results in \cref{fig:scaling-results} show that GSDM can scale gracefully to these larger problems. This is not true of our other deep learning baselines, which have performance no better than random guessing on large problem sizes. The MCMC method consistently violates less constraints than random guessing but appears to be prone to getting stuck in local minima and so seldom finds a correct solution even with small problem sizes.

\paragraph{Sorting}
Our graphical model for sorting is as follows. (1)~Sample an unsorted list $\rvu \in \mR^n$. with each element $u_i$ sampled from a unit normal. (2)~Sample a permutation matrix $\mP \in \{0,1\}^{n \times n}$. Similarly to the Sudoku case, factors on each row and column enforce that there should be a single 1 in each. (3)~Multiply $\mP$ and $\rvu$. We integrate intermediate variables $C_{ij}:=P_{ij}u_j$ and sum them as $s_i:=\sum_j C_{ij}$ to yield $\rvs=\mP\rvu$. (4)~We use factors between each pair of elements in $\rvs$ to enforce that it is sorted. 
We show a diagram of this graphical model, as well as all others used in our experiments, in \cref{app:graphical_model_specification}.
This graphical model is different to, and simpler than, our true data generation procedure in which we obtain $\rvs$ and $\mP$ with a pre-existing sorting algorithm. It fits into the GSDM framework nonetheless since the graphical model is a valid specification of the independences in the data distribution. We measure it's performance in \cref{fig:scaling-results} as the RMSE between the ground-truth $\rvs$ and the observed $\rvu$ transformed by the sampled $\mP$, and plot progress during training in \cref{fig:sorting_different_graphical_models}. Our MCMC baseline was not able to sort even 3 element lists.
\paragraph{Boolean} We additionally use the Boolean circuit described in \cref{sec:intermediate_variables} (\cref{fig:boolean-intermediate-vars-demonstration}) to demonstrate GSDM's ability to learn structured functions over many variables. Our graphical model is simply the tree-shaped compute graph of the Boolean circuit (diagram in \Cref{app:graphical_model_specification}).

\subsection{Structured Attention and Intermediate Variables}
All of our experiments rely on structured attention for their good performance, and the positive effect remains to a lesser extent after removing intermediate variables. We saw this for the Boolean circuit in \cref{sec:intermediate_variables} and here show experiments in \cref{fig:scaling-results} for Sudoku, sorting, and BCMF of various dimensionalities. Imposing structure leads to significant improvements in each case, especially in combination with intermediate variables. We ablate intermediate variables in the same figure, as well as for the Boolean circuit in \cref{fig:boolean-intermediate-vars-demonstration}. The combination of structured attention and intermediate variables is extremely helpful in all cases, ensuring that the error remains low for all tested problem dimensions, even when the baselines perform no better than random chance. 

\begin{figure}
    \centering
    \includegraphics[scale=0.6]{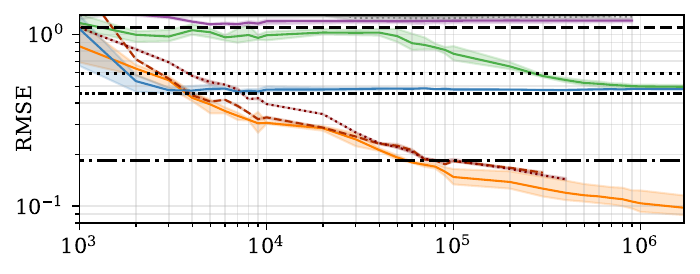}
    \includegraphics[scale=0.6]{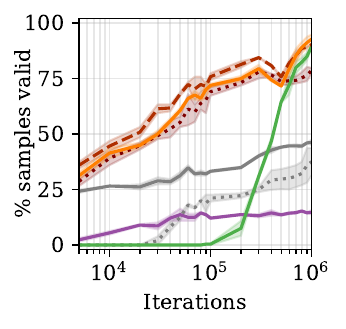}
    \includegraphics[scale=0.6]{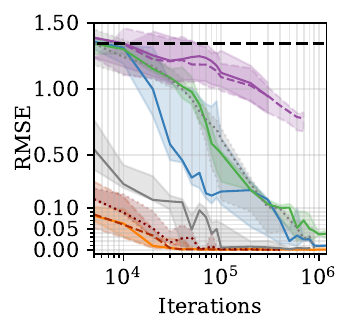}
    \caption{Training curves for BCMF (top) with $n,m,k=16,10,8$; $9\times9$ Sudoku (left); and sorting (right) with $n=20$. Error bars show min/max of 3 seeds. Legend in \cref{fig:scaling-results}.}
    \label{fig:training_plots}
\end{figure}

\subsection{Effect of Node Embeddings}
Three of our problems contain permutation invariances exploitable by embedding-sharing. Our sorting model is invariant to permutations of $j$ (the ordering of $\rvu$). Sudoku is invariant to various permutations although EE was only able to capture a simple invariance by sharing embeddings between nodes that are within both the same row and the same block. BCMF is invariant to permutations of all plate indices in \cref{fig:bmf_translation}. We see in \cref{fig:training_plots} that incorporating these invariances with EE always gives faster training than using IE. Even without specifying the invariances, AE can be used to obtain most of the benefit. For Sudoku AE outperforms IE, which may be due to AE's natural encoding of spatial information. Additionally, our results on generalization past the training dimensions in \cref{app:unconditional_bmf} are only possible with the embedding-sharing enabled by EE.

Also note that all results shown so far are for the infinite data regime, where we create data cheaply and on-the-fly during training. We show in \cref{app:data-efficiency} that our innovations are even more helpful when the number of training examples is restricted.

\subsection{Robustness to Choice of Graphical Model} \label{sec:model-structure-ablations}

There can be a degree of freedom in the choice of graphical model for a given problem. For instance, in sorting we represented the constraint that $\rvs$ is sorted with pairwise constraints between neighbouring elements. Another reasonable graphical model may have imposed a factor over all nodes of $\rvs$, making $\rvs$ fully-connected in our attention mask. We show in \cref{fig:sorting_different_graphical_models} that this choice makes negligible difference to GSDM's performance. Furthermore, some modeling choices make no difference at all to GSDM. For example we sampled $\rvu$ and $\mP$ first and then computed $\rvs:=\mP\rvu$, but someone else may have sampled $\rvs$  first (with factors to ensure it is sorted) and then $\mP$ before computing $\rvu:=\mP\rvs$. These two choices lead to identical GSDM networks because the only difference is the direction of edges in the graphical model (which is irrelevant when they are symmetrized to create the attention mask). \Cref{fig:sorting_different_graphical_models} also shows that GSDM can be robust to a misspecified model, ``Unconstrained $\rvs$'', where constraints are not imposed on $\rvs$. 
Two baselines perform significantly worse: a ``Random'' baseline, in which each node is allowed to attend to 20 other nodes sampled at random; and a baseline with a non-symmetrized version of our ``standard'' graphical model's attention mask. For sorting, edges added during symmetrization are necessary to allow any information to flow from the intermediate variables $\mC$ to the permutation matrix $\mP$. This demonstrates the necessity of our symmetrization, backing up the reasoning in \cref{sec:structured-attention}.

\begin{figure}[t]
    \centering
    \includegraphics[scale=0.7]{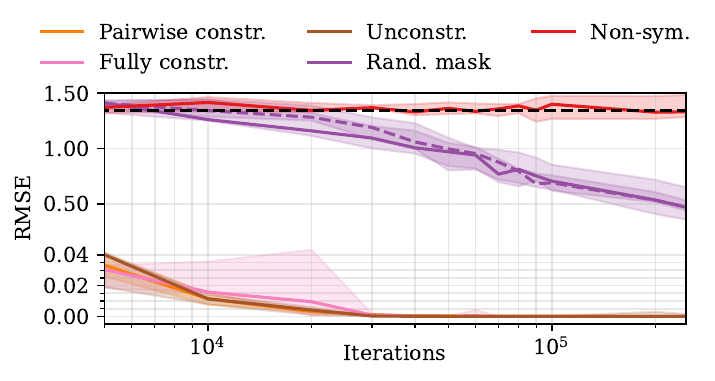}
    \caption{Ablations of GSDM with different graphical model structures for sorting with $n=20$. The first three lines, which represent reasonable or only slightly mis-specified constraints on $\rvs$ in different ways described in \cref{sec:model-structure-ablations}, quickly reach near-zero error. This speaks to GSDM's robustness to graphical model specification.
    }
    \label{fig:sorting_different_graphical_models}
    \vspace{-5mm}
\end{figure}

\section{Related Work}
\label{sec:related_work}

Sparse attention mechanisms have been introduced in several forms, either to save memory footprint 
\citep{daiTransformerXLAttentiveLanguage2019, kitaevReformerEfficientTransformer2020, royEfficientContentBasedSparse2020} or computational cost \citep{childGeneratingLongSequences2019, beltagyLongformerLongDocumentTransformer2020, zaheerBigBirdTransformers2021}. A recent review is provided in \cite{tayEfficientTransformersSurvey2022}. 

The framework of amortized inference \citep{gershmanAmortizedInferenceProbabilistic2014, ritchieDeepAmortizedInference2016} and probabilistic programming \citep{ leInferenceCompilationUniversal2017, vandemeentIntroductionProbabilisticProgramming2018} provides the foundation for our approach. But instead of requiring a full probabilistic model we relax the requirement to only specify the graphical model structure. \citet{weilbachStructuredConditionalContinuous2020} use a graphical model to structure a continuous normalizing flow for inference problems of fixed dimension. We use the more flexible and scalable DM framework including discrete variables. We can also work directly with the forward probabilistic model to avoid computing the potentially denser stochastic inverse of \citet{webbFaithfulInversionGenerative2017}.

Close in spirit to our work in terms of combinatorial optimization are \citet{selsamLearningSATSolver2019} for general SAT solving and \citet{tonshoffOneModelAny2022} for general CSP solving, which also encode the structure between variables and constraints as message passing neural networks. But these frameworks are only applicable to deterministic discrete problem classes, while we integrate everything in the more general probabilistic inference framework. \citet{freivalds2022denoising} use a DM with a graph neual network architecture~\cite{zhou2020graph} to tackle the SAT problem class. We aim to tackle a much more general class of problems. 

Our approach to explicitly training conditional diffusion models is based on that of \citet{tashiro2021csdi,harveyFlexibleDiffusionModeling2022}. Various other methods train unconditional diffusion models before providing approximate conditioning at test-time~\citep{songScoreBasedGenerativeModeling2021,hoVideoDiffusionModels2022}. Most DMs are defined over either purely continuous~\citep{hoDenoisingDiffusionProbabilistic2020a} or purely discrete spaces~\citep{austin2021structured,hoogeboom2021autoregressive}. Our approach to mixed-continuous DMs is similar to that of \citet{hoogeboom2022equivariant} but takes a variational-dequantization perspective~\citep{ho2019flow++} so that mapping back to the discrete space involves taking an $\argmax$ instead of requiring sampling.

\section{Discussion}
\label{sec:conclusion}

We have introduced GSDM which benefits from the generality of statistical conditioning, the expressivity of state-of-the-art diffusion models with attention mechanisms, and the structural reasoning applied in programming language theory and algorithm design.  We have demonstrated that GSDMs can automate the reasoning required to create approximate solutions to tasks as diverse as sorting, Sudoku solving and binary-continuous matrix factorization.  Our current implementation can be applied to graphical models with size up to roughly $50\,000$ nodes on a standard GPU and we see considerable scope for scaling further in future work. Another immediate avenue for future work to address is GSDM's relatively slow sampling speed. All samples drawn throughout this paper used 1000 network function evaluations, but recent work suggests that it may be possible to reduce this number considerably~\cite{salimans2022progressive,karras2022elucidating}. We believe that GSDM holds promise for models as complex as large scientific simulators and encourage others in the community to join us in making this a reality. In addition we are interested in integrating task-specific symbolic knowledge beyond graphical model structure into the dynamics of diffusion processes. Finally we believe that, if future work could turn our largely manual process for deriving a specific GSDM into a fully specified formal procedure, this could lead to powerful new next-generation probabilistic programming systems.

\subsubsection*{Acknowledgments}
Our thanks go to Benjamin Bloem-Reddy for a helpful discussion on equivariance in probabilistic models. 
We acknowledge the support of the Natural Sciences and Engineering Research
Council of Canada (NSERC), the Canada CIFAR AI Chairs Program, and the Intel
Parallel Computing Centers program. Additional support was provided by UBC's
Composites Research Network (CRN), and Data Science Institute (DSI). This
research was enabled in part by technical support and computational resources
provided by WestGrid (www.westgrid.ca), Compute Canada (www.computecanada.ca),
and Advanced Research Computing at the University of British Columbia
(arc.ubc.ca). WH acknowledges support by the University of British Columbia's
Four Year Doctoral Fellowship (4YF) program.

\bibliography{references}
\bibliographystyle{icml2023}

\newpage
\appendix
\onecolumn


\section{Experimental Details}
\label{app:training_parameters}

\begin{table*}[h]
    \centering
\begin{tabular}{l|c|c|c|c|c|c}
     Experiment  & Graphical model    & Conditioned on  & Struct. attn.   & Interm. vars & Disc. \& cont.    & Emb. sharing \\
     \hline
     BCMF & directed  & $\mE$     &    \checkmark    & \checkmark    & \checkmark   & (\checkmark) \\
     Sudoku & factor graph    & random subset     &    \checkmark  & -            & -           & (\checkmark) \\
     Sorting & mixed & $\vu$  &     \checkmark   & \checkmark & \checkmark  & (\checkmark) \\
     Boolean  & directed  & input        &   \checkmark & \checkmark          & -            & - 
\end{tabular}
\vspace{5pt}
        \caption{Problems tackled. A tick, \checkmark, highlights where our contributions were necessary to learn at all or scale with problem dimension and (\checkmark) where they improved performance. The improvements from structured attention and intermediate variables are shown in \cref{fig:boolean-intermediate-vars-demonstration} and \cref{fig:scaling-results}.}
    \label{tab:experiment_overview}
\end{table*}

\begin{table*}[h]
    \centering
    \caption{Experimental parameters. Listed training times refer to those used in \cref{fig:scaling-results,fig:boolean-intermediate-vars-demonstration}. The numbers of training iterations refer to those listed in the same plot with the listed problem dimension. They vary with problem dimensions as we trained all dimensions for a fixed training time on each problem, and the time per iteration depends on problem dimension. The training curves in \cref{fig:training_plots} were obtained by training for longer in some cases. The dimensions listed are those used in \cref{fig:training_plots}; dimensions are varied and clearly stated in other results. 
    }
    \begin{tabular}{llllllll}
        \toprule
        Parameter  &  Sorting  & Sudoku & BCMF & Boolean \\
        \midrule
        Problem dimension    & $n=20$     & $9\times9$  & $n,m,k=16,10,8$  & -  \\
        Training time   & 1 day  & 1 day  &  8 hours  & $40$-$160$ min. \\
        Training iters (1000s)  & $120$  & $-$   & $320$    & $20$ \\
        Batch size      & $16$         & $32$   & $8$    & $16$ \\
        Learning rate   & \num{2e-4}      & \num{2e-5} & \num{2e-5}    & \num{2e-5} \\
        Embedding dim.  & $64$            & $128$  & $64$   & $64$  \\
        \# transformer layers & $6$       & $6$     & $12$  & $12$  \\
        \# attention heads     & $8$      & $8$      & $2$   & $2$  \\
        GPU type        & A100            &  A5000  & A100         & A5000   \\
        VAEAC learning rate  &  $3\times10^{-5}$  &  $3\times10^{-4}$  &   $3\times10^{-5}$  &  $3\times10^{-5}$  \\ 
        LMH warmup samples & - & $5000$ & $5000$ & - 
    \end{tabular}
    \vspace{1em}
    \label{tab:hyperparmeter_table}
\end{table*}

We re-summarize our experiments and the results obtained in \cref{tab:experiment_overview}.
\Cref{tab:hyperparmeter_table} presents our experimental parameters. Our ablations on sorting and BCMF give all networks equal training time, and the number of iterations therefore varies depending on the time to run the network. We tuned the learning rates through small grid searches but this yielded only a small improvement to training. In keeping with common deep learning wisdom, we found that increasing the embedding dimension and number of transformer layers improved performance, as does using multiple attention heads. Conversely, the results degrade gracefully in smaller embedding dimensions, less transformer layers or varying numbers of attention heads. We set these architectural hyperparameters with the goal of obtaining networks that were both lightweight and trained quickly, and tuned them via some experimentation for sorting, Sudoku, and BCMF. We use NVIDIA A100 GPUs for sorting and BCMF, and smaller NVIDIA RTX A5000s for all ablations and other problems. We did not tune batch sizes, other than ensuring that they were large enough to obtain good GPU utilization and small enough to avoid out-of-memory errors.
All data is sampled synthetically on-the-fly, so data points used in one minibatch are never repeated in another minibatch. We use $1000$ diffusion timesteps in all experiments and set the hyperparameters $\beta_1,\ldots,\beta_{1000}$ using a linear interpolation schedule~\citep{hoDenoisingDiffusionProbabilistic2020a} from $\beta_1 = 10^{-4}$ to $\beta_{1000}=0.005$. Finally, we use the Adam optimizer with $\beta_1=0.9$ and $\beta_2=0.999$ \citep{kingmaAdamMethodStochastic2015}, no weight decay and gradient clipping at $1.0$. 
For LMH we relax the discrete delta distributions representing ``hard constraints'' for Sudoku and sorting by observing a Bernoulli variable with a probability of $0.9999$ to provide individual guidance to the sampling process for each constraint. For BCMF, we relax the Dirac distribution on the output matrix $\mE$ to a normal distribution with standard deviation $0.01$. We release code to ensure full reproducibility of our results.

For the training plots in \cref{fig:training_plots,fig:sorting_different_graphical_models,fig:hbmf_results} we compute validation metrics regularly throughout training, in particular every 1000 iterations for BCMF and HBCMF and every 5000 iterations for sorting and Sudoku. We smooth the lines by averaging the metrics for each evaluation within each segment of the logarithmic grid (so e.g. the RMSE shown for sorting at $3\times10^4$ iterations is the average of the RMSE computed at $2.5\times10^4$ and that computed at $3\times10^4$).
Where not otherwise specified, results for Sudoku are computed on examples with $50\%$ of cells observed.

\section{Graphical models}
\label{app:graphical_model_specification}

\begin{figure*}[h]
    \centering
    \begin{subfigure}[t]{0.45\textwidth}
        \centering
        \includegraphics[width=0.8\textwidth]{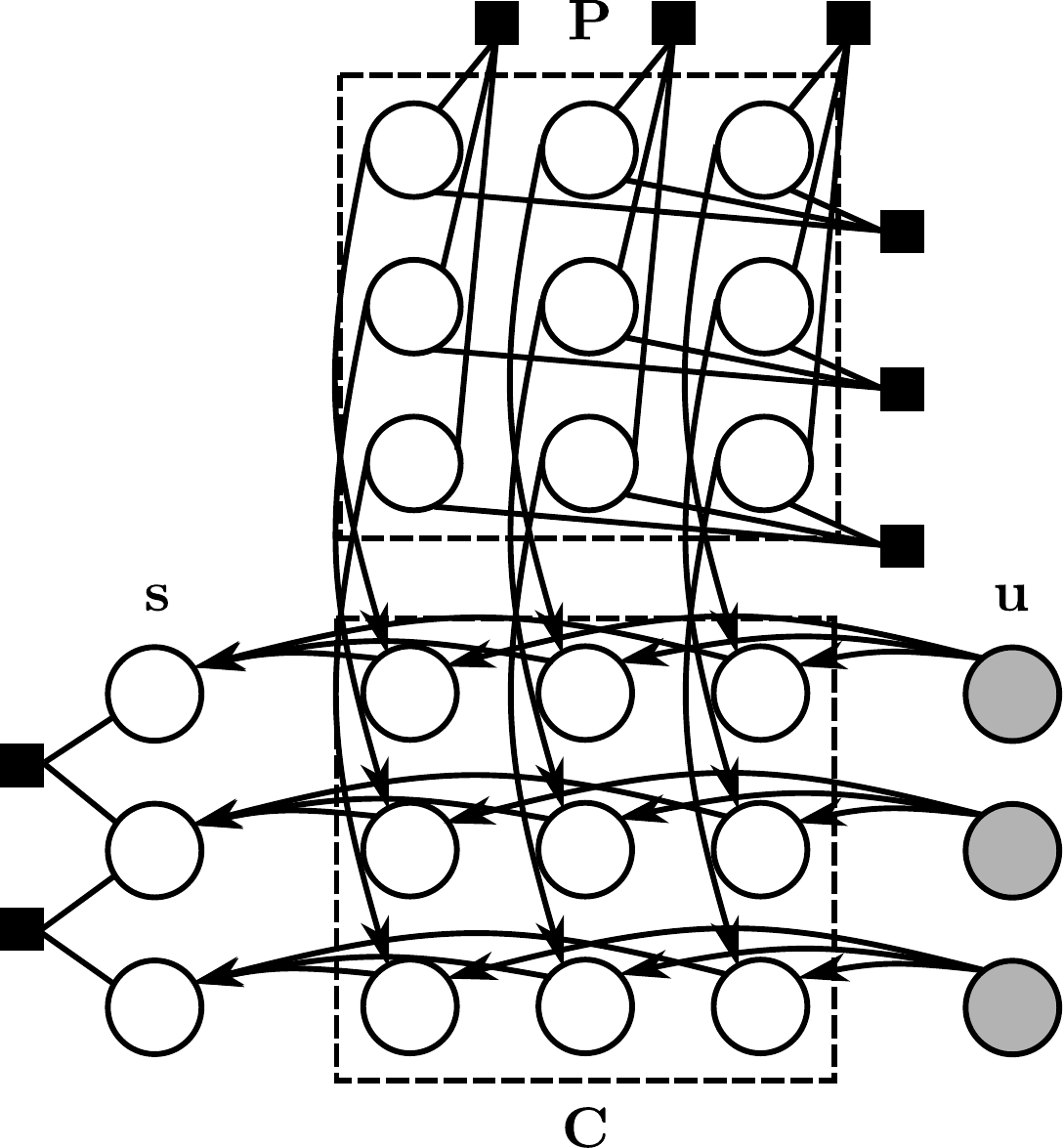}
    \caption{Graphical model for sorting. The unsorted list $\mathbf{u}$ is observed and multiplied with the permutation matrix $\mathbf{P}$ (factors ensure one element is active per row and column) into intermediate variables $\mC$. Summing over $\mC$ yields the sorted nodes $\mathbf{s}$. The sorted nodes have pairwise constraint factors to ensure their ordering.}
\label{fig:sorting_graphical_model}
    \end{subfigure}~
    \begin{subfigure}[t]{0.45\textwidth}
        \centering
        \includegraphics[width=0.7\textwidth]{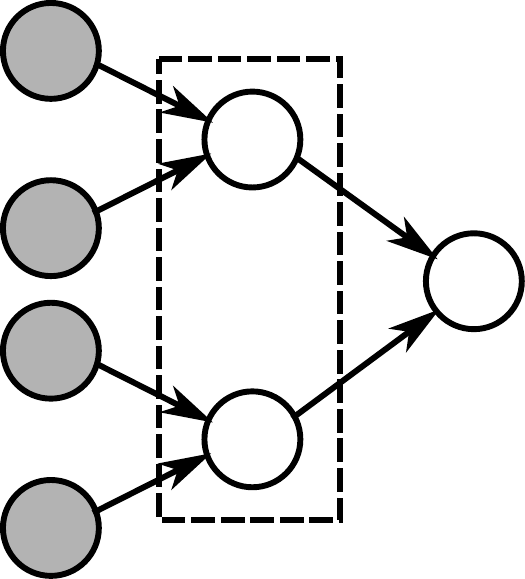}
        \caption{Boolean circuit graphical model with two layers. The nodes in the dashed box are intermediate variables. The input nodes are observed and the output node is to be inferred.}
        \label{fig:boolean-graphical-model}
    \end{subfigure}
\end{figure*}

The graphical model for BCMF is shown in \Cref{fig:bmf_translation}; for Sudoku in \Cref{fig:independence_mask_sudoku}; for sorting in \Cref{fig:sorting_graphical_model}; and for the Boolean circuit in \cref{fig:boolean-graphical-model}.
The Boolean circuit has a relatively simple tree structure. The others have more complex structures; note that even though BCMF may look like a tree in \cref{fig:bmf_translation}, the use of plates mean that it is not as there are multiple paths between e.g. elements in $\mA$ and elements in $\mE$.
The following section shows the resulting attention masks.

\section{Structured Attention}
\label{app:sparse_attention}
\label{app:attention_masks}

\Cref{fig:sparsity-masks,fig:independence_mask_sudoku} show examples of the attention masks used in all experiments, with and without intermediate variables. In \cref{fig:mask_comparison} we compare the imposed Sudoku sparsity mask to the attention weight matrices learned by our Non-sparse baseline. This confirms that, on Sudoku, the inductive bias we impose is appropriate.

\begin{figure*}[h]
\centering
\begin{subfigure}[b]{4.8cm}
    \centering
    \includegraphics[scale=1.8]{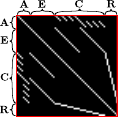}
    \caption{BCMF ($n=4,m=4,k=2$).}
\end{subfigure}
\quad
\begin{subfigure}[b]{4.8cm}
    \centering
    \includegraphics[scale=1.8]{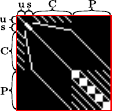}
    \caption{Sorting ($n=5$).}
\end{subfigure}
\quad
\begin{subfigure}[b]{4.8cm}
    \centering
    \includegraphics[scale=1.8]{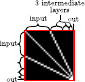}
    \caption{Boolean circuit ($n=4$).}
\end{subfigure}
\begin{subfigure}[b]{4.8cm}
    \centering
    \includegraphics[scale=1.8]{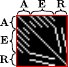}
    \caption{BCMF w/o intermediate.}
\end{subfigure}
\quad
\begin{subfigure}[b]{4.8cm}
    \centering
    \includegraphics[scale=1.8]{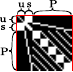}
    \caption{Sorting w/o intermediate.}
\end{subfigure}
\quad
\begin{subfigure}[b]{4.8cm}
    \centering
    \includegraphics[scale=1.8]{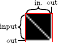}
    \caption{Boolean circuit w/o intermediate}
\end{subfigure}
\caption{Attention masks for BCMF, sorting, and the Boolean circuit. All are shown with (top row) and without (bottom row) intermediate variables. Variable $i$ can attend to variable $j$ iff the cell in row $i$ and column $j$ is white. These masks all become more sparse (as measured by proportion of entries which are non-zero) as the problem dimensionality is increased.}
\label{fig:sparsity-masks}
\end{figure*}

\begin{figure}[h]
    \centering
    \includegraphics[width=0.6\textwidth]{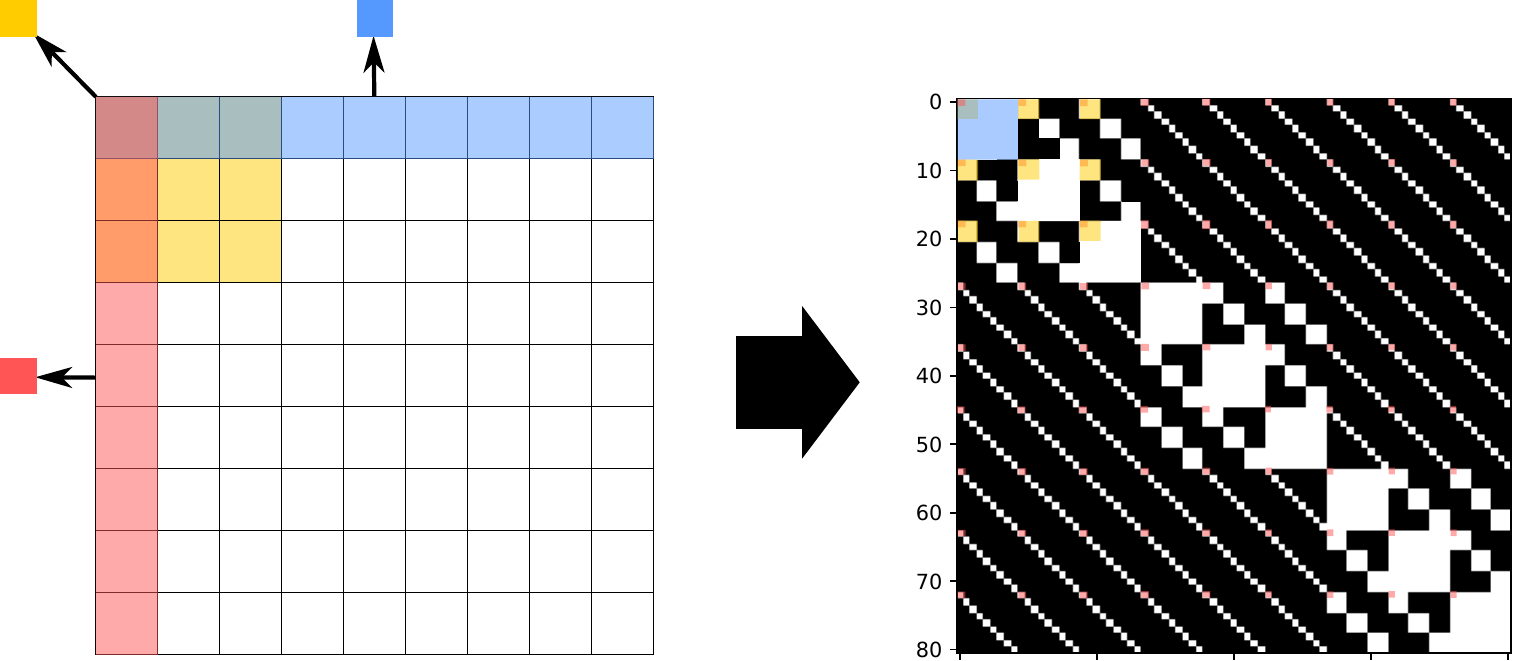}
    \caption{Correspondence between a $9\times 9$ Sudoku grid (left) and the resulting attention mask. We draw a factor for each of the first row, column and block on the left. On the right the respective entries in the attention structured mask $\mM$ are highlighted with the same color.}
    \label{fig:independence_mask_sudoku}
\end{figure}

\begin{figure}[h]
    \centering
    \includegraphics[width=0.8\textwidth]{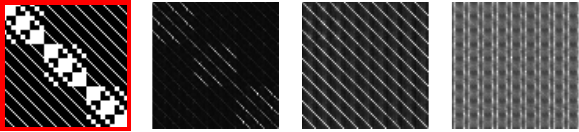}
    \caption{\textbf{Left/red border:} The attention mask imposed by GSDM on Sudoku. \textbf{Remainder:} Three attention masks obtained from different attention layers/heads of an non-sparse diffusion model on Sudoku. Soft attention to row, column and block constraints are visible, indicating that we impose an appropriate inductive bias.}
    \label{fig:mask_comparison}
\end{figure}

The optimal choice for our structured attention would be sparse matrix multiplication on the accelerator, unfortunately this was not yet available at the time of writing of this paper. We therefore provide a packed dense implementation of structured attention. 
Our structured attention mechanism lets us reduce the computational and memory cost of an $n$-dimensional DM from $\gO(n^2)$ to $\Theta(nm)$, where $m$ is the maximum number of ones in any row of our attention mask $\mM$. For our BCMF experiment, the reduction in memory footprint was necessary for us to scale to the dimensions demonstrated while using a single GPU. Recall that, after computing keys $\mK$, values $\mV$, and queries $\mQ$ (all with shape $n \times d$, where $d$ is the embedding dimension), and given a attention mask $\mM$ with shape $n \times n$, we compute
\begin{align}
    \emb &= \softmax\left( \mM \odot \mQ \mK^T \right) \mV.
\end{align}
If implemented naively with dense matrix multiplications, both computing $\mQ \mK^T$ and the outer multiplication by $\mV$ involve $\gO(n^2)$ scalar operations. We attempt to avoid this cost while still taking advantage of the dense matrix multiplications for which GPUs are designed for. To do so, we project $\mK$ and $\mV$ into 3-dimensional matrices $\mKsparse$ and $\mVsparse$ of shape $n \times m \times d$. We perform this projection such that $\mKsparse_{i}$ is a sequence of the key vectors for every variable that variable $i$ is connected to. Equivalently, letting $a_{ij}$ be the index of the $j$th variable that variable $i$ is connected to, $\mKsparse_{i,j}$ is equal to $\mK_{a_{ij}}$. We define $\mVsparse$ similarly for value vectors. If variable $i$ connects to less than $m$ entries, then we pad $\mKsparse_i$ and $\mVsparse_i$ with zeros. We then compute an $n\times m$ array of unnormalised weights $\mU$ (encompassing all interactions allowed by our attention mask) such that $\mU_{i,j} = \mQ_i \cdot \mKsparse_{i,j}$. Doing so involves only $\gO(nm)$ operations, rather than the $\gO(n^2)$ required to compute dense attention weights. We then mask all entries in $\mU$ that were padded by setting them to $-\infty$ before applying the usual $\text{softmax}(\mU)$ row-wise to get $\mWsparse$. Finally, we can compute the output $\evh$ by setting each $\emb_i = \sum_j{\mWsparse_{i,j} \mVsparse_j}$ (again requiring only $\gO(nm)$ operations), which is equal to the output that would be obtained through dense matrix multiplications including the mask $\mM$. Future work may integrate the sparse attention implementation of \citet{kreuzer2021rethinking} that scales as $\gO(e)$.

\section{Neural network architecture}
\label{app:neural-net-architecture}

In the following we denote Swish activations as $\nnSilu$ \citep{hendrycksGaussianErrorLinear2020}, group norm as $\nnNorm$, 2D convolution as $\nnConv$, a linear layer as $\nnLinear$ and neuron dropout as $\nnDropout$. Following the implementation of \citet{songDenoisingDiffusionImplicit2021} our timestep embedding is implemented as the following sequence: $\nnLinear$; $\nnSilu$; $\nnLinear$. Our ResNet also follows \citet{songDenoisingDiffusionImplicit2021} by applying the sequence: $\nnNorm$; $\nnSilu$; $\nnConv$; projection of time and variable embedding; $\nnNorm$; $\nnSilu$; $\nnDropout$; $\nnConv$. Compared to \citet{songDenoisingDiffusionImplicit2021} our convolution here is 1D and not 2D since we do not operate on images.

We repurpose the same architecture for our ``Regressor + GS'' and ``Regressor'' baselines, using the GSDM architecture for ``Regressor + GS'' and the non-sparse variation for ``Regressor''. To do so, we simply replace the inputs $\rvx_t$ and $t$ with appropriately-shaped arrays of zeros for every training iteration as well as during inference. The modified architectures therefore take a sole input $\rvy$ and return a sole output, their prediction of $\rvx$.

\section{Faithfulness of Attention}
Our design choices regarding the construction of attention masks, and the depth of our architecture, are motivated by the following. According to the DM loss in \cref{eq:diffusion-loss}, the neural network $\hat{\rvx}_\theta$ is tasked at each timestep $t=t'$ with predicting $\rvx_0$ from $\rvx_{t'}$. \Cref{fig:graph_dm} shows an example graphical model factorization for $q(\rvx_0|\rvx_{t'})$ corresponding to this task. Since $\rvx_{t'}$ can be generated by adding independent noise to each dimension of $\rvx_0$, this graphical model is derived from the graphical model of the data distribution $q(\rvx_0)$ by simply adding an edge from each variable in $\rvx_0$ to the corresponding variable in $\rvx_{t'}$.

\begin{theorem}[Dependence in diffusion models]
\label{theo:cond_dep_dm}
Given that the data distribution $q(\rvx_0)$ is represented by a connected graphical model $\gG = (\gX, \gA)$, with nodes $\gX$ and edges $\gA$, we can represent the temporally combined graphical model at times $t=0$ and $t=t'$ as $\gG_{DM} = (\{\rvx_{t'}^i \}_i \cup \{ \rvx_0^i \}_i, \{(\rvx_0^i , \rvx_{t'}^i)\}_i \cup \gA)$. Then there are no pairs $i,j$ such that $\rvx_0^i$ can be assumed independent of $\rvx_{t'}^j$ after conditioning on all other dimensions of $\rvx_{t'}$. In other words, for any $i,j$ pair, we have to assume $\rvx_0^i \not\!\perp\!\!\!\perp \rvx_{t'}^j \mid \rvx_{t'}^{-j}$, where $\rvx_{t'}^{-j}$ stands for all nodes in $\rvx_{t'}$ except node $j$.
\end{theorem}

\begin{proof}
For any $j$, $\rvx_{t'}^{j}$ is directly connected to $\rvx_{0}^{j}$. Since we assumed that the graphical model for $q(\rvx_0)$ is connected, there will further be a path from $\rvx_{0}^{j}$ to $\rvx_{0}^{i}$ which does not pass through any conditioned on nodes for any $i$. Therefore $\rvx_{t'}^{j}$ cannot be d-separated from $\rvx_{0}^{i}$~\citep{kollerProbabilisticGraphicalModels2009a}.
\end{proof}

\label{app:faithfulness}
\begin{wrapfigure}[13]{r}{0.2\textwidth}
    \includegraphics[width=0.2\columnwidth]{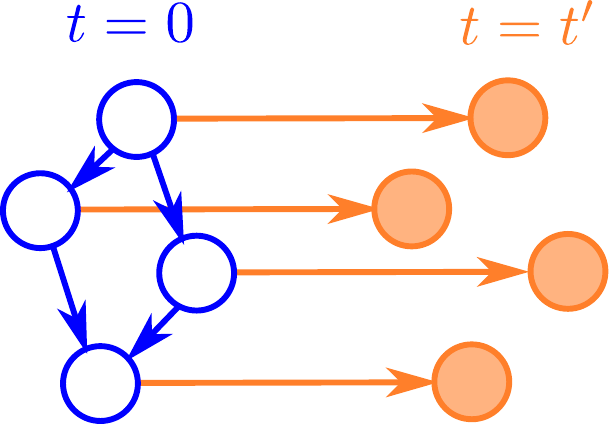}
    \caption{Example graphical model of $q(\rvx_0)q(\rvx_t|\rvx_0)$. Nodes are latent/blue if in $\rvx_0$, observed/orange if in $\rvx_t$.}
    \label{fig:graph_dm}
\end{wrapfigure}
The consequence of \cref{theo:cond_dep_dm} is that every node in the neural network output $\hat{\rvx}_\theta(\rvx_t,\rvy,t)$ should depend on every node in its input $\rvx_t$.  Neural network architectures without this property might not be able to faithfully predict $\rvx_0$ given $\rvx_t$ and so might not faithfully model $q(\rvx_0)$. This consideration motivated our previously-described design choice that variable $i$ can attend to $j$ if there is any edge between them, irrespective of the direction of the edge. Otherwise, if the graphical model is directed and acyclic there will not be a path between every pair of nodes. This would cause GSDM to make false independence assumptions, which we show impacts performance in \cref{fig:sorting_different_graphical_models}. Note that if node $i$ is not directly connected to node $j$ in our attention mask, information about node $i$ may have to be passed to node $j$ via other nodes. Since messages are only passed along one edge per transformer layer, the number of transformer layers should be chosen to be at least as great as the maximum path length in the symmetrized graphical model.

\section{Permutation invariance through shared embeddings} \label{app:permutation-invariance}
In this section we first provide intuition about permutation invariance in diffusion models before restating and proving \cref{theo:perm-invariance-gsdm}. Our architecture provides an opportunity to enforce permutation invariance in the learned distribution. Consider a ``vanilla'' DM which has a similar architecture to ours in \cref{fig:attention_circuit} but without the attention mask. The only information this network receives about the ordering of its input comes from the node embeddings and, if the node embeddings were removed, the architecture would be entirely permutation-equivariant. Similarly, if the node embeddings were shared between a set of nodes, the architecture would be equivariant to permutations of this set of nodes. 
The modeled distribution would therefore be invariant to permutations of this set of nodes~\cite{hoogeboom2022equivariant}. 
This may be a useful permutation invariance to encode for some problems but, for the structured problems considered in this paper, it is too simple and not valid. We introduce the following theorem to get describe other, more practically applicable, types of permutation invariance.

\begin{theorem}[Permutation invariance in GSDM]
Let $\mathcal{A}$ represent the indices of a subset of the dimensions of data $\rvx$ and $\Pi_\mathcal{A}$ be the class of permutations that permute only dimensions indexed by $\mathcal{A}$. Assume we have a GSDM parameterised with neural network $\hat{\rvx}_\theta(\cdot; \mM)$, where $\mM$ is the structured attention mask. If the node embeddings used by $\hat{\rvx}_\theta$ are shared across all nodes indexed by $\mathcal{A}$, then the distribution modelled by GSDM will be invariant to all permutations $\pi$ satisfying
\begin{equation}
    \mM = \pi \mM \quad \text{and} \quad \pi \in \Pi_\mathcal{A}
\end{equation}
where $\pi\mM$ is a permutation of both the rows and columns of $\mM$ by $\pi$.
\end{theorem}
\begin{proof}
Our architecture without sparse attention, i.e. with $\mM=\mathbf{1}$, is equivariant under $\Pi_\mathcal{A}$ in that (writing the attention mask as an additional input)
\begin{equation}
    \hat{\rvx}_\theta(\pi \rvx_t; \mathbf{1}) = \pi \hat{\rvx}_\theta(\rvx_t; \mathbf{1}) \qquad \forall \pi \in \Pi_\mathcal{A}
    \label{eq:equivariant-dm}
\end{equation}
for any $\rvx_t$. The analysis is different when the network uses an attention mask because the attention mask provides additional information about the ordering of the inputs. Replacing $\mM$ by $\pi\mM$, the permutation of both the rows and columns of $\mM$ by $\pi$, prevents this:
\begin{equation}
    \hat{\rvx}_\theta(\pi \rvx_t, \pi \mM) = \pi \hat{\rvx}_\theta(\rvx_t, \mM) \qquad \forall \pi \in \Pi_\mathcal{A}
    \label{eq:non-equivariant-gsdm}
\end{equation}
For equivariance to permutations of $\rvx$ alone, however, we require
\begin{equation}
    \hat{\rvx}_\theta(\pi \rvx_t, \mM) = \pi \hat{\rvx}_\theta(\rvx_t, \mM) \qquad \forall \pi \in \Pi_\mathcal{A}.
    \label{eq:equivariant-gsdm}
\end{equation}
In general, the equality in \cref{eq:non-equivariant-gsdm} will only imply that in \cref{eq:equivariant-gsdm} if $\mM = \pi \mM$. Therefore, when used in combination with a structured attention mask $\mM$, sharing embeddings among all nodes in $\mathcal{A}$ will lead to equivariance only to the set of permutations $\{\pi\in\Pi_\mathcal{A}|\mM=\pi\mM\}$. 
\end{proof}

\section{Generalization with BCMF problem dimension}
\label{app:unconditional_bmf}
\label{app:conditional_bcmf}

\begin{figure*}[h]
    \centering
    \includegraphics[width=1.0\textwidth]{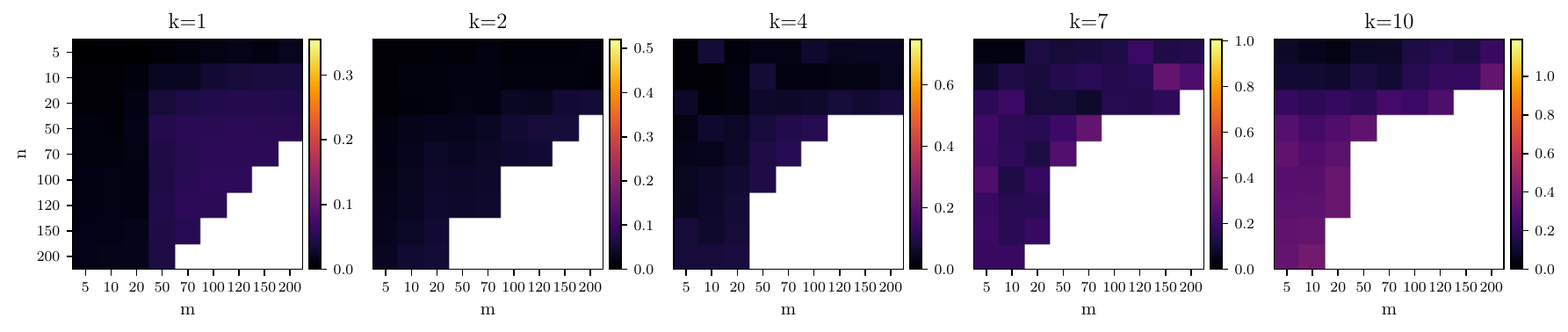}
    \caption{Error vs problem dimension for GSDM on binary continuous matrix factorization. We show the root mean square error between the observed matrix $\mE$ and the product of the sampled $\mA$ and $\mR$. The colorbar for each value of $k$ is scaled so that ``yellow'' corresponds to the error achieved by a baseline which samples $\mA$ and $\mR$ from the prior, ignoring $\mE$. Despite never seeing a value of $n$, $m$, or $k$ larger than $10$ during training, GSDM scales well to much larger values of $m$ and $n$. When they grow large enough, GSDM runs out of GPU memory. We mark entries where this occurred in white.}
    \label{fig:bmf_rank_comparison}
\end{figure*}

\begin{figure}[h]
    \centering
    \includegraphics[width=1.0\textwidth]{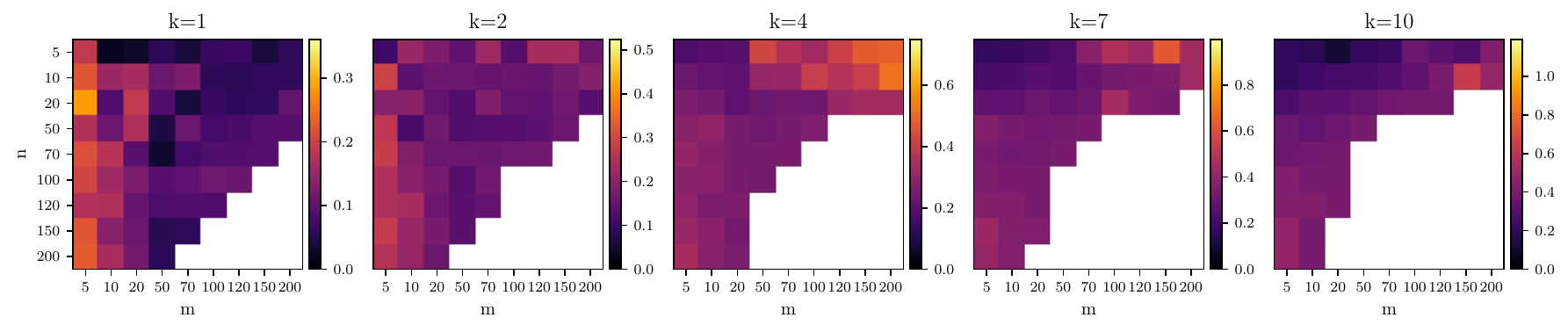}
    \caption{Similar to \Cref{fig:bmf_rank_comparison} but with a GSDM model trained to sampled $\mA$, $\mR$, and $\mE$ jointly instead of conditioning on $\mE$. We plot a heatmap of the root mean squared error (RMSE) between the matrix $\mE$ and the product $\mA\mR$. The ranges for each rank are scaled so that a yellow color represents the expected RMSE if $\mA$, $\mR$, and $\mE$ are all sampled independently from the prior. Entries colored in white exceed the GPU's memory limit.}
    \label{fig:bmf_rank_comparison_unconditioned}
\end{figure}

In \Cref{fig:bmf_rank_comparison_unconditioned} we plot a heatmap similar to \Cref{fig:bmf_rank_comparison} but with a GSDM which generates unconditional joint samples of $\mA$, $\mR$, and $\mE$ instead of samples conditioned on $\mE$. We measure the mismatch between $\mE$ and the product $\mA\mR$ and, interestingly, see that it is greater for the unconditional model (for problem dimensions both inside and outside the training distribution). This suggests that it may be possible to improve unconditional generation performance by adjusting the diffusion process hyperparameters so that $\mE$ is sampled early in the diffusion process and then $\mA$ and $\mR$ are sampled later, conditioned on $\mE$, but we do not attempt to do so here.

\begin{table*}[]
    \centering
    \begin{tabular}{c|c|ccc|c}
         & & Cost of ResNets  &  Cost of attention & Overall cost  &  Reduction \\
         \midrule
         \multirow{2}{*}{BCMF with $k=m=n$} & Naive & $\gO(n^2)$ & $\gO(n^4)$  & $\gO(n^4)$ & - \\
         & GSDM &  $\gO(n^3)$ & $\gO(n^4)$  & $\gO(n^4)$ & $\gO(1)$ \\
         \midrule
         \multirow{2}{*}{Sorting with input size $n$} & Naive & $\gO(n^2)$ & $\gO(n^4)$  & $\gO(n^4)$ & - \\
         & GSDM &  $\gO(n^2)$ & $\gO(n^3)$  & $\gO(n^3)$ & $\gO(n)$ \\
         \midrule
         \multirow{2}{*}{Sudoku with side length $n$} & Naive & $\gO(n^2)$ & $\gO(n^4)$  & $\gO(n^4)$ & - \\
         & GSDM &  $\gO(n^2)$ & $\gO(n^3)$  & $\gO(n^3)$ & $\gO(n)$ \\
         \midrule
         \multirowcell{2}[0pt][l]{Boolean with input size $n$} & Naive & $\gO(n)$ & $\gO(n^2)$  & $\gO(n^2)$ & - \\
         & GSDM & $\gO(n)$ & $\gO(n)$  & $\gO(n)$ & $\gO(n)$  \\
         \midrule
    \end{tabular}
    \caption{Comparison of computational complexities of a GSDM layer and a naive DM layer without structured attention or intermediate variables. GSDM yields reductions in complexity that scale with $n$ for the Boolean and sorting experiments, while giving the same complexity as a naive approach for BCMF yet much better performance.}
    \label{tab:complexities}
\end{table*}

\section{Computational Complexities} \label{ap:complexities}
\Cref{tab:complexities} compares the computational cost of GSDM with that of naively applying a DM without intermediate variables or structured attention.

\section{Data efficiency}
\label{app:data-efficiency}

Throughout this paper, we have focused on the case where infinite training data is available from a fast-to-sample generative model of $\rvx$ and $\rvy$. We now show that GSDM provides further advantages if limited data is available through the inductive biases that we impose on its architecture. In \cref{fig:varying-dataset-size} we demonstrate this by showing results when we train on a finite number of examples. On the left-hand side of each plot, the number of unique data points is equal to the batch size so that all data points are seen in every batch. Our non-sparse baseline does consistently worse than GSDM in this setting, and the GSDM variations with non-exchangeable embeddings have performance in between. GSDM reaches its best performance in each case with less than a thousand data points. The non-sparse baseline needs considerably more to be competitive. Training times and hyperparameters for each variation were as mentioned in \cref{tab:hyperparmeter_table}.

\begin{figure}
    \centering
    \begin{subfigure}[t]{0.2\textwidth}
    \includegraphics[scale=0.7]{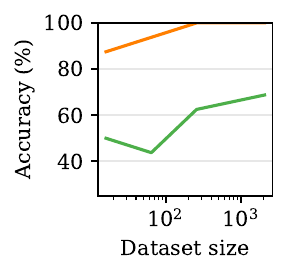}
    \caption{Boolean circuit}
    \end{subfigure}
    \begin{subfigure}[t]{0.2\textwidth}
    \includegraphics[scale=0.7]{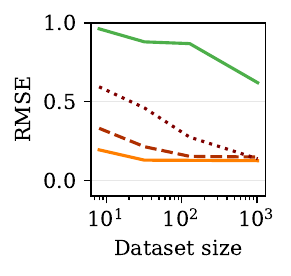}
    \caption{BCMF}
    \end{subfigure}
    \begin{subfigure}[t]{0.2\textwidth}
    \includegraphics[scale=0.7]{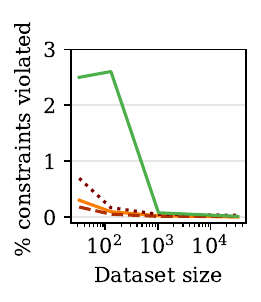}
    \caption{Sudoku}
    \end{subfigure}
    \begin{subfigure}[t]{0.2\textwidth}
    \includegraphics[scale=0.7]{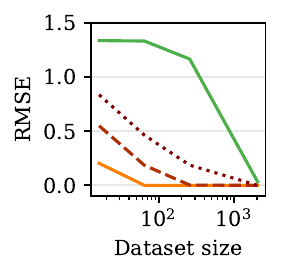}
    \caption{Sorting}
    \end{subfigure}
    \begin{subfigure}[t]{0.15\textwidth}
    \vspace{-3cm}
    \includegraphics[scale=0.7]{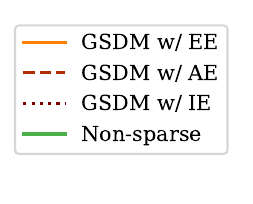}
    \end{subfigure}
    \caption{RMSE for GSDM and baselines trained on varying numbers of unique data points. GSDM needs far fewer data points to than the baselines to fit well.}
    \label{fig:varying-dataset-size}
\end{figure}

\section{Sudoku sample diversity}
\label{app:sudoku_sample_diversity}

\begin{figure}[h]
    \centering
    \includegraphics[width=0.5\columnwidth]{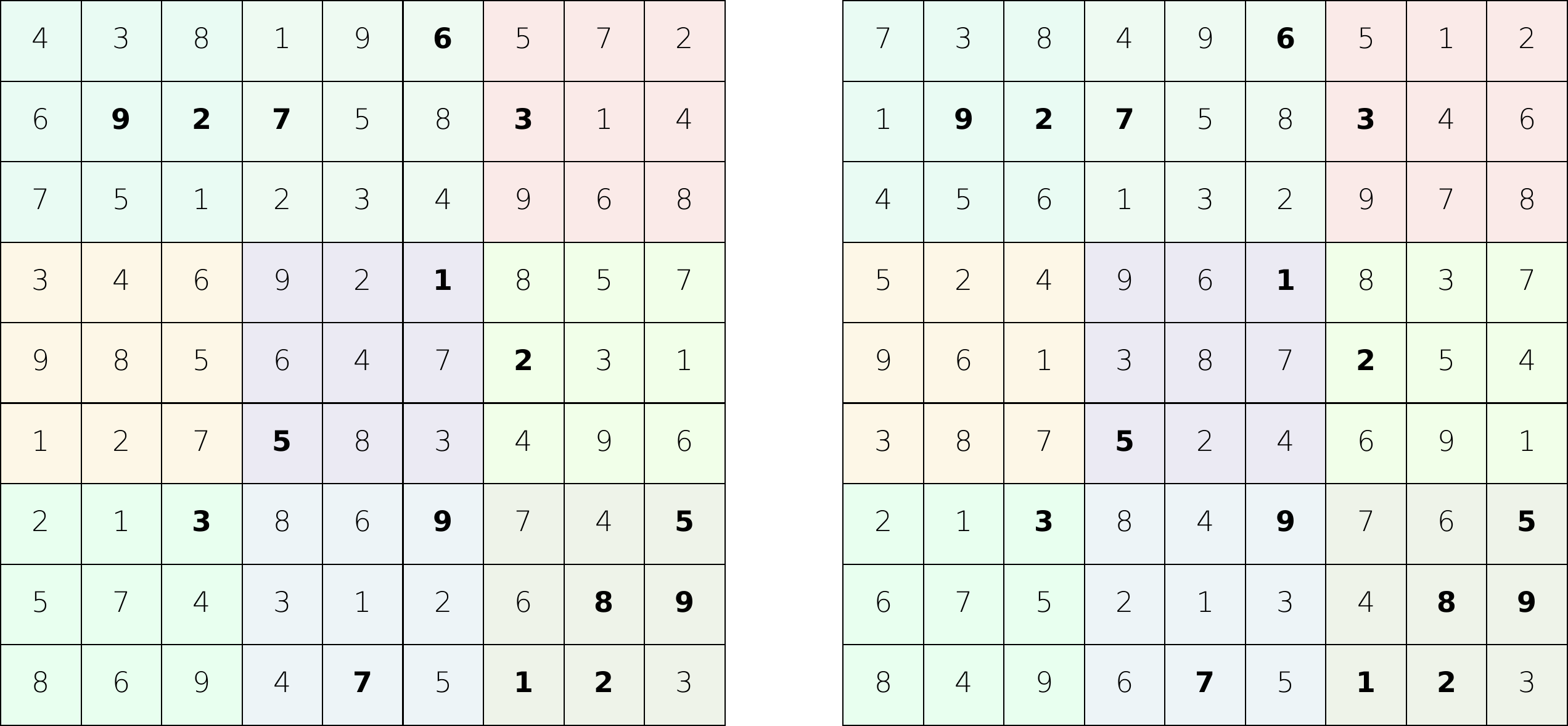}
    \caption{Two Sudoku solutions conditioned on the same 16 observed cells (in bold).}
    \label{fig:sudoku_panel}
\end{figure}

Sample diversity for Sudoku solving is shown in \Cref{fig:sudoku_panel}. GSDM's objective is naturally mass-covering so enforces sample diversity. 

\section{Matrix inversion}
\label{app:matrix_inversion}

\begin{figure}[h]
    \centering
    \includegraphics{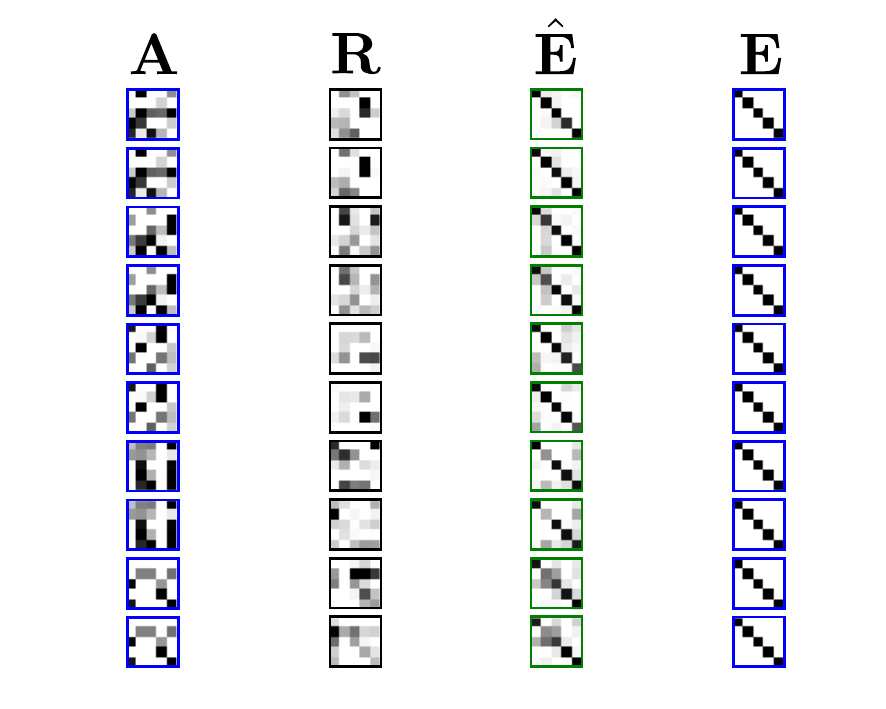}
    \caption{Rows of matrix inversion examples. Similar to \cref{fig:bmf_reconstruction}, but here we condition both on $\mA$ and $\mE=\mathbf{1}$ (blue). Each of the 5 pairs of rows show a solution for the same $\mA$. Reconstructions are shown as $\hat{\mE}=\mA \mR$ (green).}
    \label{fig:matrix_inversion}
\end{figure}

In this experiment we explore a purely continuous variation of our BCMF example from \Cref{sec:experiments}. We train the model on fixed size full rank matrices of dimension $5$ and condition on both $\mE$ and $\mA$ during training and testing. All entries of both $\mA \in \sR^{5 \times 5}$ and $\mR \in \sR^{5 \times 5}$ are now sampled from a $\text{Normal}(0,1)$ prior and we set $\mE = \mA \mR$ (as in BCMF). Despite training on randomly sampled $\mA$ and $\mR$, we demonstrate that GSDM implicitly learns matrix inversion. At test time we set $\mE$ to the identity matrix and solve for $\mR \approx \mA^{-1}$. Example solutions can be seen in \cref{fig:matrix_inversion}. Each pair of rows contains two approximate solutions for the same $\mA$ to illustrate sample diversity. Most reconstructions for $\hat{\mE}$ are close to the identity matrix, but GSDM is not perfect. We did not specialize our prior from BCMF; a more targeted prior could be constructed by directly providing pairs of matrices and their inverse. This experiment shows that we are able to calculate approximate inverses, even though we have not specialized our graphical model or training distribution to do so.

\section{Automatic Compilation of BCMF}
\label{app:bmf_compilation}

\begin{wrapfigure}[10]{r}{0.25\textwidth}
    \centering
    \vspace{-.5cm}
    \includegraphics[width=0.15\textwidth]{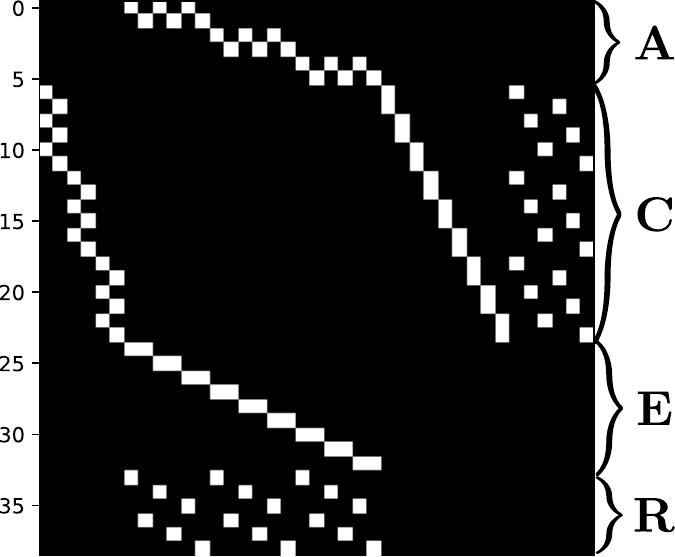}
    \caption{Connectivity mask extracted from BCMF source code. This is the same structure as in \Cref{fig:bmf_translation} but with permuted indices and before the addition of the diagonal self-edges.}
    \label{fig:bmf_compiled_mask}
\end{wrapfigure}
Building on the probabilistic programming language defined in~\citet{vandemeentIntroductionProbabilisticProgramming2018}, we demonstrate a compiler which maps from programs into a corresponding graphical model structure. We demonstrate it on the program on page \pageref{fig:bcmf_source}, which multiplies two random matrices $\mA \in \mathbb{R}^{3 \times 2}, \mR \in \mathbb{R}^{2 \times 3}$ similarly to our BCMF experiment. Samples from Dirac distributions are used to introduce the intermediate nodes of $\mC$ and the terminal nodes of $\mE$. Our compiler first translates it into a graphical model and then into the attention mask as shown in \Cref{fig:bmf_compiled_mask}. We envisage a future extension which ``compiles'' directly from such source code to a trained GSDM network.

\section{BCMF baselines}
\label{app:bcmf-baselines}
Here we provide full detail on the ``hand-coded algorithm'' baselines which we compare against on BCMF. They are as follows

\paragraph{K-means} We treat the $m \times n$ matrix $\mE$ as a dataset containing $n$ data points, each made up of $m$-dimensions.  We then use an off-the-shelf implementation of K-means clustering with number of clusters $K$ set to to the rank $k$. This returns an $m \times k$ matrix of cluster centers, which we use as $\mA$. It also returns the cluster indices in $\{1,\ldots,k\}$ associated with each of the $n$ data points. We convert the indices to one-hot vectors and stack them to give a $k \times n$ binary matrix, which we return as $\mR$.  Doing a factorization of $\mE$ into $\mA$ and $\mR$ in this manner and then reconstructing $\hat{E}:=\mA\mR$ can be understood as snapping each row of $\mE$ to one of the $k$ ``mean'' rows in $\mR$. Therefore, when $k$ is equal to $n$, this approach yields zero error, as is visible in \cref{fig:scaling-results}. The error quickly increases as $n$ grows larger than $k$.

\paragraph{NMF} We perform non-negative matrix factorization of $\mE \in \mathbb{R}^{m \times n}$ into $\tilde{\mA} \in \mathbb{R}^{m \times k}$ and $\tilde{\mR} \in \mathbb{R}^{k \times n}$  with coordinate descent, as implemented in Scikit-learn~\citet{pedregosa2011scikit}. We then simply round each element in $\mR$ to be in $\{0, 1\}$.

\paragraph{ChatGPT}
We show the prompt and produced source code in \cref{fig:chatgpt-baseline}. The algorithm it produces finds the $k$ largest rows in $\mE$ and copies them to $\mA$, adding an index in $\mR$ so that they are reconstructed if we compute $\hat{E}:=\mA\mR$. The reconstruction is thus perfect for the largest $k$ rows and zero for all other rows. This means that the baseline achieves zero error when $k = n$ and relatively large error otherwise. We also see in \cref{fig:scaling-results} that the error spikes for $k$ larger than $n$ and $m$ (e.g. for $n=m=16$ and $k=20$) as this edge case is unhandled.

\newpage
\begin{figure}
\begin{lstlisting}
(defn rand-matrix [size name]
  (foreach (first size) [i (range (first size))]
           (foreach (second size) [j (range (second size))]
                    (sample name (normal 0 1)))))

(defn dot-helper [t state a b]
  (+ state
     (sample "C" (dirac (* (get a t)
                           (get b t))))))

(defn dot [a b]
  (loop (count a) 0 dot-helper a b))

(defn row-mul [t state m v]
  (conj state (sample "E" (dirac (dot (get m t) v)))))

(defn transpose [m]
  (foreach (count (first m)) [j (range (count (first m)))]
           (foreach (count m) [i (range (count m))]
                    (get (get m i) j))))

(defn matmatmul [m1 m2]
  (let [m2_ (transpose m2)]
    (foreach (count m1) [i (range (count m1))]
             (foreach (count m2_) [j (range (count m2_))]
                      (sample
                       "E"
                       (dirac (dot (get m1 i) (get m2_ j))))))))

(let [A (rand-matrix [3 2] "A")
      R (rand-matrix [2 3] "R")
      E (matmatmul A R)]
  E)
\end{lstlisting}
\caption{Source code of a full generative model for the BCMF experiment. Passing this into our compiler yields the attention mask in \cref{fig:bmf_compiled_mask}. Note that intermediate variables for $\mC$ are explicitly created by sampling from a dirac distribution.}
\label{fig:bcmf_source}
\end{figure}

\newpage

\begin{figure}

\textbf{Prompt:} Write a Python function which takes as input E, an m x n non-negative matrix, and factorises it into a mxk binary matrix and a kxn continuous component with values in [0, 1]
\\

\textbf{Returned code:}
{\tiny
\begin{lstlisting}
import numpy as np

def factorize(E, k):
    """
    Factorize an m x n non-negative matrix E into a mxk binary matrix and a kxn continuous component with values in [0, 1].
    
    Args:
        E (numpy.ndarray): The non-negative matrix to be factorized.
        k (int): The number of columns in the binary matrix and rows in the continuous component.
    
    Returns:
        A tuple containing the binary matrix and continuous component, respectively.
    """
    
    # Initialize the binary and continuous matrices
    B = np.zeros((E.shape[0], k), dtype=np.int)
    C = np.zeros((k, E.shape[1]), dtype=np.float)
    
    # Perform the factorization
    for i in range(k):
        # Find the row in E with the largest sum that hasn't been assigned to a column in B yet
        max_sum = -1
        max_row = -1
        for j in range(E.shape[0]):
            if np.sum(B[j,:]) == 0 and np.sum(E[j,:]) > max_sum:
                max_sum = np.sum(E[j,:])
                max_row = j
        
        # Assign the row to the ith column of B
        B[max_row, i] = 1
        
        # Compute the ith row of C as a weighted average of the rows in E selected by B
        selected_rows = B[:,i] == 1
        C[i,:] = np.sum(E[selected_rows,:], axis=0) / np.sum(selected_rows)
    
    return B, C
\end{lstlisting}
}
\caption{Binary-continuous matrix factorization code produced by ChatGPT~\citep{chatgpt} We used the prompt at the top and show the code produced, ignoring additional description and examples that it gave.}
\label{fig:chatgpt-baseline}
\end{figure}

\section{Hierarchical BCMF}
\label{sec:hierarchical-bcmf}

In this section we present an extension of the binary-continuous matrix factorization model in which the generative model samples factors from a hierarchical prior. As shown in \cref{fig:hbmf_graphical_model}, our generative model samples a rank variable and then sample the factors and conditioned on having exactly this many non-zero rows. We report the generative model, the associated sparsity mask, our results, and example factorizations. The differing sizes of and of the factorizations reflect the different ranks that are sampled for each. We show results in \cref{fig:hbmf_results,fig:hbmf_samples}. GSDM consistently selects an appropriate rank. We find that, as in our non-hierarchical matrix factorization experiment, GSDM is the best-performing method and indeed the only neural method to outperform our hand-crafted approximate solutions.

\begin{figure}[h]
    \centering
    \includegraphics[width=0.15\textwidth]{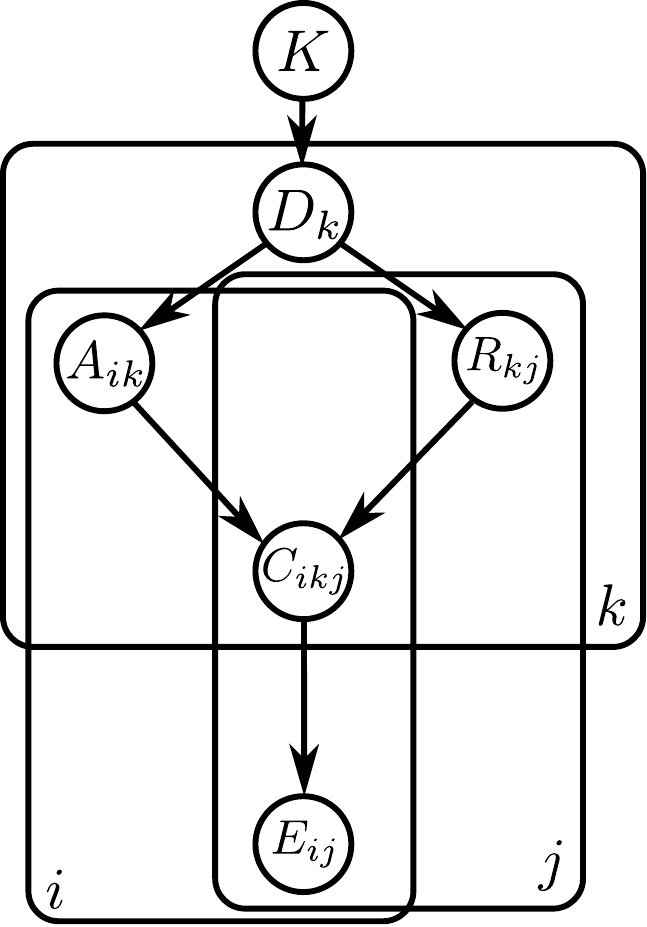}
    \includegraphics[width=0.24\textwidth]{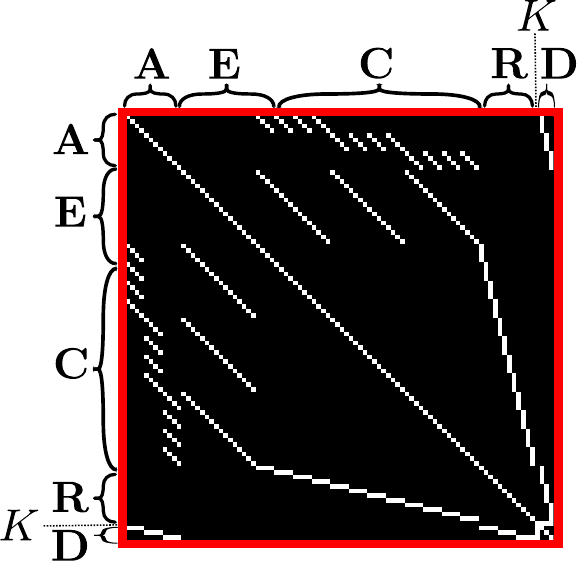}
    \caption{\textbf{Left:} A graphical model for hierarchical binary-continuous matrix factorization (HBCMF). The rank $K$ is first sampled from a uniform distribution. Given this, $\mathbf{D}_k$ is set to 1 for $k \leq K$ and 0 otherwise. Each $k$th row of $\mathbf{R}$ is then sampled conditioned on being zero iff $\mathbf{D}_k$. The same is done for each $k$th column of $\mathbf{A}$.  \textbf{Right:} The sparsity mask for GSDM derived from the HBCMF graphical model. It is similar to that for BCMF but with added rows and columns for $K$ and $\mathbf{D}$. The single dimension of $K$ is indicated with a dashed line. In this model, the upper limit on $K$ is three so $\mathbf{D}$ has three dimensions. $\mathbf{D}$ interacts with both $\mathbf{A}$ and $\mathbf{R}$ following the graphical model. }
    \label{fig:hbmf_graphical_model}
\end{figure}

\begin{figure}[h]
    \centering
    \vspace{-.5cm}
    \includegraphics[width=0.6\textwidth]{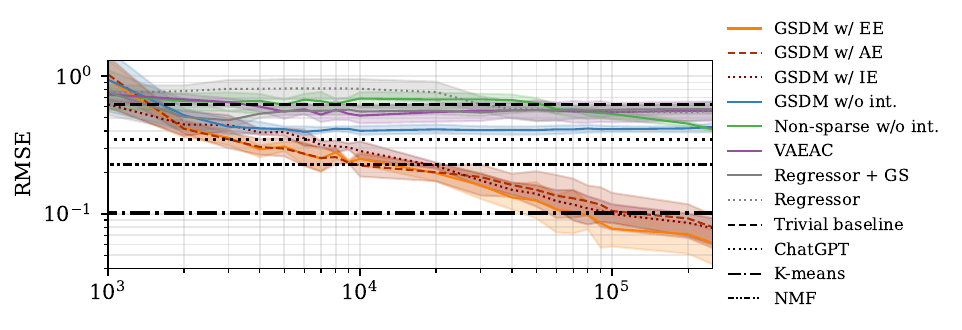}
    \caption{Experimental results on HBCMF. As in our non-hierarchical matrix factorization experiment, GSDM is the best-performing method and indeed the only neural method to outperform our hand-crafted approximate solutions.}
    \label{fig:hbmf_results}
\end{figure}

\begin{figure}[h]
    \centering
    \vspace{-.5cm}
    \includegraphics[width=0.7\textwidth]{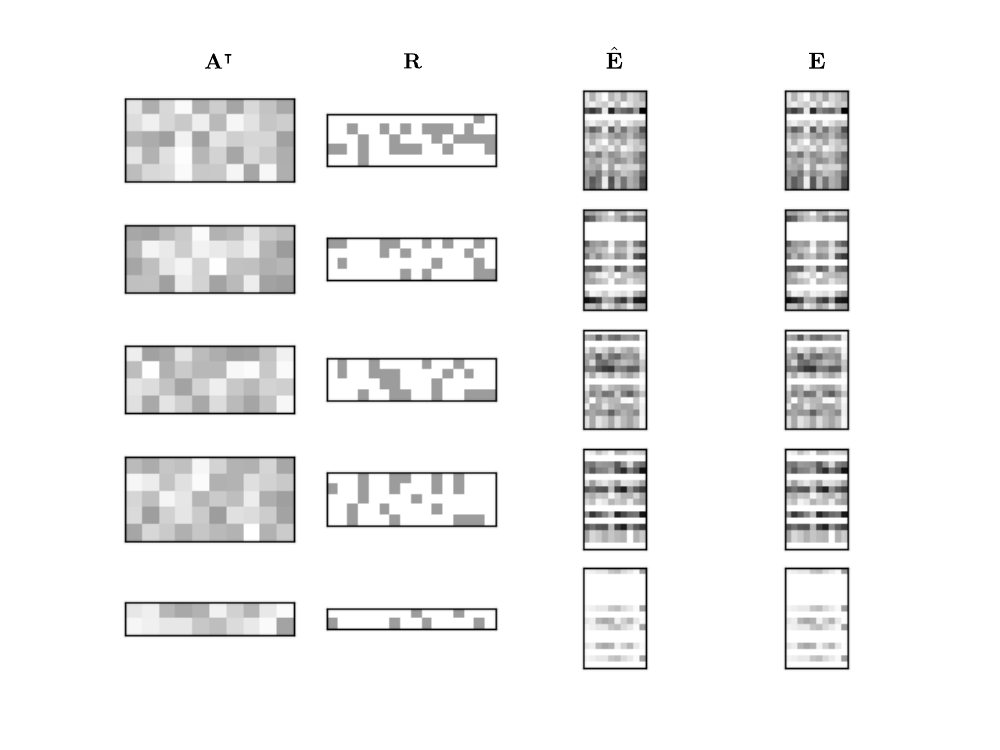}
    \vspace{-.5cm}
    \caption{Example HBCMF factorizations of five matrices by GSDM, one per row. The differing sizes of $\mathbf{A}$ and $\mathbf{R}$ in this figure reflect the different ranks that are sampled for each. GSDM consistently selects an appropriate rank.}
    \vspace{-.5cm}
    \label{fig:hbmf_samples}
\end{figure}

\end{document}

%% file: math_commands.tex
\usepackage{amsmath,amsfonts,bm}

\def\eqref#1{equation~\ref{#1}}

\def\1{\bm{1}}

\def\rvu{{\mathbf{i}}}

\def\rvs{{\mathbf{s}}}

\def\rvu{{\mathbf{u}}}

\def\rvx{{\mathbf{x}}}
\def\rvy{{\mathbf{y}}}

\def\rmI{{\mathbf{I}}}

\def\ve{{\bm{e}}}

\def\vu{{\bm{u}}}

\def\evh{{h}}

\def\mA{{\bm{A}}}

\def\mC{{\bm{C}}}

\def\mE{{\bm{E}}}

\def\mK{{\bm{K}}}

\def\mM{{\bm{M}}}

\def\mP{{\bm{P}}}
\def\mQ{{\bm{Q}}}
\def\mR{{\bm{R}}}

\def\mU{{\bm{U}}}
\def\mV{{\bm{V}}}
\def\mW{{\bm{W}}}

\DeclareMathAlphabet{\mathsfit}{\encodingdefault}{\sfdefault}{m}{sl}
\SetMathAlphabet{\mathsfit}{bold}{\encodingdefault}{\sfdefault}{bx}{n}

\def\gA{{\mathcal{A}}}

\def\gG{{\mathcal{G}}}

\def\gN{{\mathcal{N}}}
\def\gO{{\mathcal{O}}}

\def\gX{{\mathcal{X}}}

\def\sR{{\mathbb{R}}}

\newcommand{\R}{\mathbb{R}}

\newcommand{\softmax}{\mathrm{softmax}}

\DeclareMathOperator*{\argmax}{arg\,max}

\def\emb{{\ve}}

\def\mKsparse{\Bar{\mK}}
\def\mVsparse{\Bar{\mV}}
\def\mWsparse{\Bar{\mW}}

\def\nnNorm{\text{norm}}
\def\nnSilu{\text{SiLU}}
\def\nnLinear{\text{linear}}
\def\nnConv{\text{conv2d}}
\def\nnDropout{\text{dropout}}